\documentclass{article}
\PassOptionsToPackage{numbers, compress}{natbib}

\usepackage[final]{neurips_2021}

\usepackage[utf8]{inputenc} 
\usepackage[T1]{fontenc}    
\usepackage{xcolor}
\usepackage{hyperref}       
\hypersetup{     
    colorlinks=true,  
    breaklinks=true,                
    urlcolor=blue,                
    linkcolor=red,                
    bookmarksopen=false,
    filecolor=black,
    citecolor=blue,
    linkbordercolor=blue
}
\usepackage{url}            
\usepackage{booktabs}       

\usepackage{amsmath,amsfonts,bm,amssymb,amsthm}
\usepackage{thmtools,thm-restate}
\usepackage{mathtools}
\usepackage{optidef}
\usepackage{breqn}
\usepackage{xfrac}

\makeatletter
\def\thm@space@setup{\thm@preskip=4pt
\thm@postskip=-2pt}
\makeatother
\newtheoremstyle{newstyle}      
{} 
{} 
{\itshape} 
{} 
{\bfseries} 
{.} 
{ } 
{} 

\theoremstyle{newstyle}

\newcommand*\diff{\mathrm{d}}
\newcommand*\NN{\hat{\mathcal{F}}_{\tiny\text{NN}}}
\newcommand*\er{\mathrm{er}}
\newtheorem{prop}{Proposition}
\newtheorem{lemma}{Lemma}
\newtheorem{theorem}{Theorem}
\newtheorem{cor}{Corollary}
\usepackage[mathcal]{euscript}
\newcommand*\ess{\operatorname{ess}}

\let\originalleft\left
\let\originalright\right
\renewcommand{\left}{\mathopen{}\mathclose\bgroup\originalleft}
\renewcommand{\right}{\aftergroup\egroup\originalright}

\def\Figref#1{Fig.~\ref{#1}}
\def\twofigref#1#2{figures \ref{#1} and \ref{#2}}

\def\Secref#1{Sec.\,\ref{#1}}

\def\Suppref#1{Sup.\,\ref{#1}}
\def\eqref#1{eq.\,\ref{#1}}
\def\Eqref#1{Eq.\,\ref{#1}}

\def\Tabref#1{Table~\ref{#1}}

\def\twotabref#1#2{tables \ref{#1} and \ref{#2}}

\def\ceil#1{\left\lceil #1 \right\rceil}
\def\floor#1{\left\lfloor #1 \right\rfloor}
\def\1{\bm{1}}

\def\eps{{\varepsilon}}

\def\gTh{{\hat \gT}}

\def\gPh{{\hat \gP}}

\def\rmx{{\mathbf{x}}}
\def\rmy{{\mathbf{y}}}

\def\mF{{\bm{F}}}
\def\mG{{\bm{G}}}

\def\mQ{{\bm{Q}}}

\def\mLambda{{\bm{\Lambda}}}

\DeclareMathAlphabet{\mathsfit}{\encodingdefault}{\sfdefault}{m}{sl}
\SetMathAlphabet{\mathsfit}{bold}{\encodingdefault}{\sfdefault}{bx}{n}

\def\gA{{\mathcal{A}}}

\def\gC{{\mathcal{C}}}

\def\gF{{\mathcal{F}}}
\def\gG{{\mathcal{G}}}
\def\gGh{{\hat\gG}}
\def\gFh{{\hat\gF}}

\def\gH{{\mathcal{H}}}

\def\gL{{\mathcal{L}}}

\def\gN{{\mathcal{N}}}

\def\gP{{\mathcal{P}}}

\def\gT{{\mathcal{T}}}

\def\sH{{\mathbb{H}}}

\newcommand{\Lm}{\mathrm{L}}

\newcommand{\R}{\mathbb{R}}

\usepackage{amsfonts}       
\usepackage{nicefrac}       
\usepackage{microtype}      
\usepackage{graphicx}
\usepackage{multicol}
\usepackage{array}
\usepackage{xspace}
\usepackage{wrapfig}
\usepackage{xhfill}
\usepackage{multirow}
\usepackage{subfig}
\usepackage{float}
\usepackage[inline]{enumitem}
\usepackage[resetlabels,labeled]{multibib}
\newcites{S}{References}
\usepackage{tikz}

\makeatletter
\renewcommand\maketitle{\par
  \begingroup
    \@maketitle
    \@thanks
    \@notice
    \thispagestyle{empty}
  \endgroup
  \setcounter{footnote}{0}%
}
\makeatother

\usepackage{xparse}
\let\oriCiteS\citeS
\RenewDocumentCommand{\citeS}{O{} O{} m}{%
  \renewcommand{\citenumfont}[1]{S##1}%
  \oriCiteS[#1][#2]{#3}%
  \renewcommand{\citenumfont}[1]{##1}%
}

\newcommand{\camerareadychange}[1]{\textcolor{blue}{#1}}
\renewcommand{\camerareadychange}[1]{{#1}}

\title{\textit{LEADS}: Learning Dynamical Systems that Generalize Across Environments}

\author{%
  Yuan~Yin\textsuperscript{1}, Ibrahim~Ayed\textsuperscript{1,2}, Emmanuel~de~Bézenac\textsuperscript{1}, Nicolas Baskiotis\textsuperscript{1}, Patrick Gallinari\textsuperscript{1,3}\\
  \textsuperscript{1}Sorbonne Université, Paris, France \\
  \textsuperscript{2}ThereSIS Lab, Thales, Paris, France ~~
  \textsuperscript{3}Criteo AI Lab, Paris, France \\
  \texttt{\{yuan.yin,ibrahim.ayed,emmanuel.de-bezenac,}\\\texttt{nicolas.baskiotis,patrick.gallinari\}@sorbonne-universite.fr}
}

\begin{document}
\linepenalty=1000
\setlength{\abovedisplayskip}{2.5pt}
\setlength{\belowdisplayskip}{2.5pt}
\everypar{\looseness=-1}

\maketitle

\let\model\textit
\def\RQone{\model{RQ1}}
\def\RQtwo{\model{RQ2}}

\def\OFA{\model{One-For-All}\@\xspace}
\def\OPE{\model{One-Per-Env.}\@\xspace}
\def\OPES{\model{One-Per-Env.(single)}\@\xspace}
\def\LEADS{\model{LEADS}\@\xspace}
\def\LEADSNM{\LEADS\model{no min.}\@\xspace}
\def\PRT{\model{Pre-trained-$f$-Only}\@\xspace}
\def\LEADSBS{\model{Pre-trained-$f$-Plus-Trained-$g_e$}\@\xspace}

\newcommand*{\eg}{e.g.\@\xspace}
\newcommand*{\ie}{i.e.\@\xspace}
\newcommand*{\iid}{i.i.d.\@\xspace}
\newcommand*{\wrt}{w.r.t.\@\xspace}
\newcommand*{\etc}{etc.\@\xspace}
\newcommand*{\cf}{cf.\@\xspace}

\begin{abstract}
When modeling dynamical systems from real-world data samples, the distribution of data often changes according to the environment in which they are captured, and the dynamics of the system itself vary from one environment to another. Generalizing across environments thus challenges the conventional frameworks. The classical settings suggest either considering data as \iid and learning a single model to cover all situations or learning environment-specific models. Both are sub-optimal: the former disregards the discrepancies between environments leading to biased solutions, while the latter does not exploit their potential commonalities and is prone to scarcity problems. We propose \LEADS, a novel framework that leverages the commonalities and discrepancies among known environments to improve model generalization. This is achieved with a tailored training formulation aiming at capturing common dynamics within a shared model while additional terms capture environment-specific dynamics. We ground our approach in theory, exhibiting a decrease in sample complexity \wrt classical alternatives. We show how theory and practice coincides on the simplified case of linear dynamics. Moreover, we instantiate this framework for neural networks and evaluate it experimentally on representative families of nonlinear dynamics. We show that this new setting can exploit knowledge extracted from environment-dependent data and improves generalization for both known and novel environments.
\end{abstract}

\section{Introduction}
Data-driven approaches offer an interesting alternative and complement to physical-based methods for modeling the dynamics of complex systems and are particularly promising in a wide range of settings: \eg if the underlying dynamics are partially known or understood, if the physical model is incomplete, inaccurate, or fails to adapt to different contexts, or if external perturbation sources and forces are not modeled. The idea of deploying machine learning (ML) to model complex dynamical systems picked momentum a few years ago, relying on recent deep learning progresses and on the development of new methods targeting the evolution of temporal and spatiotemporal systems~\cite{brunton2016discovering,BezenacPG18, ChenRBD2018,LongLMD2018,RaissiPK2019,AyedDPBG2019,YinLDDATG2021}. It is already being applied in different scientific disciplines~(see \eg \cite{Willard2020} for a recent survey) and could help accelerate scientific discovery to address challenging domains such as climate~\cite{Reichstein2019} or health~\cite{Fresca2020}.

However, despite promising results, current developments are limited and usually postulate an idealized setting where data is \textit{abundant} and \textit{the environment does not change}, the so-called ``\iid hypothesis''. In practice, real-world data may be expensive or difficult to acquire. Moreover, changes in the environment may be caused by many different factors. For example, in climate modeling, there are external forces~(\eg Coriolis) which depend on the spatial location~\cite{Nemo}; or, in health science, parameters need to be personalized for each patient as for cardiac computational models~\cite{NeicCPNBVP2017}. More generally, data acquisition and modeling are affected by different factors such as geographical position, sensor variability, measuring circumstances, etc. The classical paradigm either considers all the data as \iid and looks for a global model, or proposes specific models for each environment. The former disregards discrepancies between the environments, thus leading to a biased solution with an averaged model which will usually perform poorly. The latter ignores the similarities between environments, thus affecting generalization performance, particularly in settings where per-environment data is limited.
This is particularly problematic in dynamical settings, as small changes in initial conditions lead to trajectories not covered by the training data.

In this work, we consider a setting where it is explicitly assumed that the trajectories are collected from different environments. Note that in this setting, the \iid hypothesis is removed twice: by considering the temporality of the data and by the existence of multiple environments. In many useful contexts the dynamics in each environment share similarities, while being distinct which translates into changes in the data distributions. Our objective is to leverage the similarities between environments in order to improve the modeling capacity and generalization performance, while still carefully dealing with the discrepancies across environments. This brings us to consider two research questions:
\begin{description}[nosep,topsep=0pt,leftmargin=24pt]
    \item [\RQone] Does modeling the differences between environments improve generalization error \wrt classical settings: \textbf{\OFA}, where a unique function is trained for all environments; and \textbf{\OPE}, where a specific function is fitted for each environment? (\cf \Secref{sec:exp} for more details)
    \item [\RQtwo] Is it possible to extrapolate to a novel environment that has not been seen during training? 
\end{description}
We propose LEarning Across Dynamical Systems~(\LEADS), a novel learning methodology decomposing the learned dynamics into \textit{shared} and \textit{environment-specific} components. The learning problem is formulated such that the \textit{shared} component captures the dynamics common across environments and exploits all the available data, while the \textit{environment-specific} component only models the remaining dynamics, \ie those that cannot be expressed by the former, based on environment-specific data. We show, under mild conditions, that the learning problem is well-posed, as the resulting decomposition exists and is unique~(\Secref{sec:learning across env}). We then analyze the properties of this decomposition from a sample complexity perspective. While, in general, the bounds might be too loose to be practical, a more precise study is conducted in the case of linear dynamics for which theory and practice are closer. We then instantiate this framework for more general hypothesis spaces and dynamics, leading to a heuristic for the control of generalization that will be validated experimentally. Overall, we show that this framework provides better generalization properties than \OPE, requiring less training data to reach the same performance level~(\RQone). The shared information is also useful to extrapolate to unknown environments: the new function for this environment can be learned from very little data~(\RQtwo). We experiment with these ideas on three representative cases~(\Secref{sec:exp}) where the dynamics are provided by differential equations: ODEs with the Lotka-Volterra predator-prey model, and PDEs with the Gray-Scott reaction-diffusion and the more challenging incompressible Navier-Stokes equations. Experimental evidence confirms the intuition and the theoretical findings: with a similar amount of data, the approach drastically outperforms \OFA and \OPE settings, especially in low data regimes. \camerareadychange{Up to our knowledge, it is the first time that generalization in multiple dynamical systems is addressed from an ML perspective\footnote{Code is available at \url{https://github.com/yuan-yin/LEADS}.}.}

\vspace{-6pt}
\section{Approach\label{sec:approach}}
\vspace{-6pt}
\subsection{Problem setting\label{subsec:problem_setting}\label{sec:setting}}
\vspace{-6pt}
We consider the problem of learning models of dynamical physical processes with data acquired from a set of environments $E$. Throughout the paper, we will assume that the dynamics in an environment $e\in E$ are defined through the evolution of differential equations. This will provide in particular a clear setup for the experiments and the validation. For a given problem, we consider that the dynamics of the different environments share common factors while each environments has its own specificity, resulting in a distinct model per environment. Both the general form of the differential equations and the specific terms of each environment are assumed to be completely unknown. $x^e_t$ denotes the state of the equation for environment $e$, taking its values from a bounded set $\gA$, with evolution term $f_e: \gA \to T\gA$, $T\gA$ being the tangent bundle of $\gA$. In other words, over a fixed time interval $[0, T]$, we have:
\begin{equation}
    \label{eq:dyn}
    \frac{\diff x^e_t}{\diff t}= f_e(x^e_t)
\end{equation}
We assume that, for any $e$, $f_e$ lies in a functional vector space $\gF$. In the experiments, we will consider one ODE, in which case $\gA \subset \R^d$, and two PDEs, in which case $\gA$ is a $d'$-dimensional vector field over a bounded spatial domain $S \subset \R^{d'}$. The term of the data-generating dynamical system in \Eqref{eq:dyn} is sampled from a distribution for each $e$, \ie $f_e \sim Q$. From $f_e$, we define $\gT_e$, the data distribution of trajectories $x^e_\cdot$ verifying \Eqref{eq:dyn}, induced by a distribution of initial states $x_0^e \sim P_0$. The data for this environment is then composed of $l$ trajectories sampled from $\gT_e$, and is denoted as $\smash{\gTh_e}$ with $\smash{x_{\cdot}^{e,i}}$ the $i$-th trajectory. We will denote the full dataset by $\smash{\gTh = \bigcup_{e \in E} \gTh_e}$. 

\looseness=-100
The classical empirical risk minimization (ERM) framework suggests to model the data dynamics either at the global level (\OFA), taking trajectories indiscriminately from $\smash{\gTh}$, or at the specific environment level (\OPE), training one model for each $\smash{\gTh_e}$. Our aim is to formulate a new learning framework with the objective of explicitly considering the existence of different environments to improve the modeling strategy \wrt the classical ERM settings.

\vspace{-6pt}
\subsection{\textit{LEADS} framework\label{sec:learning across env}}
\vspace{-6pt}
\looseness=-1
We decompose the dynamics into two components where $f \in \gF$ is shared across environments and $g_e\in \gF$ is specific to the environment $e$, so that
\begin{equation}
    \forall e\in E, f_e = f+g_e \label{eq:constraint}
\end{equation}
Since we consider functional vector spaces,  this additive hypothesis is not restrictive and such a decomposition always exists. It is also quite natural as a sum of evolution terms can be seen as the sum of the forces acting on the system. Note that the sum of two evolution terms can lead to behaviors very different from those induced by each of those terms. However, learning this decomposition from data defines an ill-posed problem: for any choice of $f$, there is a $\smash{\{g_e\}_{e\in E}}$ such that \Eqref{eq:constraint} is verified. A trivial example would be $f=0$ leading to a solution where each environment is fitted separately.

Our core idea is that $f$ should capture as much of the shared dynamics as is possible, while $g_e$ should focus only on the environment characteristics not captured by $f$. To formalize this intuition, we introduce $\Omega(g_e)$, a penalization on $g_e$, which precise definition will depend on the considered setting. 
We reformulate the learning objective as the following constrained optimization problem:
\begin{equation}
\min\limits_{f,\{g_e\}_{e\in E}\in\gF} ~~~ \sum_{e\in E}\Omega(g_e) ~~~
\mathrm{subject~to} ~~~
\forall x^{e,i} \in\gTh, \forall t,  \frac{\diff x_t^{e,i}}{\diff t} =(f+g_e)(x_t^{e,i})\label{eq:opt}
\end{equation}
Minimizing $\Omega$ aims to reduce $g_e$s' complexity while correctly fitting the dynamics of each environment. This argument will be made formal in the next section. Note that $f$ will be trained on the data from all environments contrary to $g_e$s.
A key question is then to determine under which conditions the minimum in \Eqref{eq:opt} is well-defined. The following proposition provides an answer~(proof cf. \Suppref{supp:proof}):
\begin{restatable}[Existence and Uniqueness]{prop}{propone}
\label{prop:exist_unique}
Assume $\Omega$ is convex, then the existence of a minimal decomposition $\smash{{f^\star, \{g^\star_e\}_{e\in E} \in \gF}}$ of \Eqref{eq:opt} is guaranteed. Furthermore, if $\Omega$ is strictly convex, this decomposition is unique.
\end{restatable}
In practice, we consider the following relaxed formulation of Eq. \ref{eq:opt}:
\begin{equation}
\min\limits_{f,\{g_e\}_{e\in E}\in\gF} ~~~
\sum_{e\in E}\: \bigg(\frac{1}{\lambda} \Omega(g_e) + \sum_{i=1}^{l}\int_0^T\bigg\|
\frac{\mathrm{d}x_t^{e,i}}{\mathrm{d}t}
- (f + g_e)(x^{e,i}_\tau)
\bigg\|^2 \mathrm{d}t\bigg)\label{eq:lagrangian}
\end{equation}
where $f,g_e$ are taken from a hypothesis space $\smash{\gFh}$ approximating $\smash{\gF}$. $\lambda$ is a regularization weight and the integral term constrains the learned $f+g_e$ to follow the observed dynamics. The form of this objective and its effective calculation will be detailed in \Secref{subsec:training-details}. 

\vspace{-6pt}
\section{Improving generalization with \textit{LEADS} \label{sec:sample_complexity}}
\vspace{-6pt}
Defining an appropriate $\Omega$ is crucial for our method. In this section, we show that the generalization error should decrease with the number of environments. While the bounds might be too loose for NNs, our analysis is shown to adequately model the decreasing trend in the linear case, linking both our intuition and our theoretical analysis with empirical evidence. This then allows us to construct an appropriate $\Omega$ for NNs.

\vspace{-6pt}
\subsection{General case\label{subsec:general_case}}
\vspace{-6pt}
After introducing preliminary notations and definitions, we define the hypothesis spaces associated with our multiple environment framework. Considering a first setting where all environments of interest are present at training time, we prove an upper-bound of their effective size based on the covering numbers of the approximation spaces. This allows us to quantitatively control the sample complexity of our model, depending on the number of environments $m$ and other quantities that can be considered and optimized in practice. We then consider an extension for learning on a new and unseen environment. The bounds here are \camerareadychange{inspired by ideas initially introduced in \cite{Baxter2000}. They consider multi-task classification in vector spaces, where the task specific classifiers share a common feature extractor. Our extension considers sequences corresponding to dynamical trajectories, and a model with additive components instead of function composition in their case.
}
\vspace{-6pt}
\paragraph{Definitions.}
Sample complexity theory is usually defined for supervised contexts, where for a given input $x$ we want to predict some target $y$. In our setting, we want to learn trajectories $\smash{(x_t^e)_{0\leq t \leq T}}$ starting from an initial condition $x_0$. We reformulate this problem and cast it as a standard supervised learning problem: $\gT_e$ being the data distribution of trajectories for environment $e$, as defined in \Secref{subsec:problem_setting}, let us consider a trajectory $\smash{x^e_\cdot \sim \gT_e}$, and time $\tau \sim \mathrm{Unif}([0, T])$; we define system states $\smash{x\!=\!x^e_{\tau} \in \gA}$ as input, and the corresponding values of derivatives $\smash{y=f_e(x^e_\tau) \in T\gA}$
as the associated target. We will denote $\gP_e$ the underlying distribution of $(x,y)$, and $\smash{\gPh_e}$ the associated dataset of size $n$. 

We are searching for $f, g_e\!:\!\gA\!\to\!T\gA$ in an approximation function space $\smash{\gFh}$ of the original space $\gF$. Let us define $\smash{\gGh \subseteq \gFh}$ the effective function space from which the $g_e$s are sampled. Let $f+\smash{\gGh}:=\{f + g: g\in \smash{\gGh}\}$ be the hypothesis space generated by function pairs $(f, g)$, with a fixed $\smash{f \in \gFh}$. For any $h: \gA\to T\gA$, the error on some test distribution $\gP_e$ is given by $\smash{\er_{\gP_e}(h)} = \smash{\int_{\gA\times T\gA} \|h(x) - y\|^2\diff \gP_e(x, y)}$ and the error on the training set by $\smash{\hat{\er}_{\gPh_e}(h) = \frac{1}{n}\sum_{(x,y)\in\gPh_e}\|h(x)-y\|^2}$.%

\vspace{-6pt}
\paragraph{\LEADS sample complexity.} Let $\smash{\gC_\gGh(\varepsilon, \gFh)}$ and $\smash{\gC_\gFh(\varepsilon, \gGh)}$ denote the capacity of $\smash{\gFh}$ and $\smash{\gGh}$ at a certain scale $\eps > 0$. Such capacity describes the approximation ability of the space. The capacity of a class of functions is defined based on covering numbers, and the precise definition is provided in \Suppref{supp:sample_complexity_general}, \Tabref{tab:def}. The following result is general and applies for \textit{any} decomposition of the form $f + g_e$. It states that to guarantee a given average test error, the minimal number of samples required is a function of both capacities and the number of environments  $m$, and it provides a step towards \RQone~(proof see \Suppref{supp:sample_complexity_general}):
\begin{restatable}{prop}{proptwo}
\label{prop:sample_complexity_all_envs}
Given $m$ environments, let $\eps_1, \eps_2, \delta > 0, \eps = \eps_1+\eps_2$. Assume the number of examples $n$ per environment satisfies
\begin{equation}
n\geq \max\bigg\{\frac{64}{\eps^2}\bigg(\frac{1}{m} \Big(\log \frac{4}{\delta} + \log \gC_\gGh\Big(\frac{\eps_1}{16}, \gFh\Big)\Big)  + \log \gC_\gFh\Big(\frac{\eps_2}{16}, \gGh\Big)\bigg), \frac{16}{\varepsilon^2}\bigg\} \label{eq:known-envs}
\end{equation}
Then with probability at least $1-\delta$ (over the choice of training sets $\smash{\{\gPh_e\}}$), any learner $(f+g_1, \dots, f+g_m)$ will satisfy 
\(\smash{\frac{1}{m}\sum_{e\in E}\er_{\gP_e}(f+g_e) \leq  \frac{1}{m}\sum_{e\in E}\hat{\er}_{\gPh_e}(f+g_e) + \varepsilon}\).
\end{restatable}
The contribution of $\smash{\gFh}$ to the sample complexity decreases as $m$ increases, while that of $\smash{\gGh}$ remains the same: this is due to the fact that shared functions $f$ have access to the data from all environments, which is not the case for $g_e$. From this finding, one infers the basis of \LEADS: when learning from several environments, to control the generalization error through the decomposition $f_e=f+g_e$, \emph{$f$ should account for most of the complexity of $f_e$ while the complexity of $g_e$ should be controlled and minimized}. We then establish an explicit link to our learning problem formulation in \Eqref{eq:opt}. Further in this section, we will show for linear ODEs that the optimization of $\Omega(g_e)$ in \Eqref{eq:lagrangian} controls the capacity of the effective set $\smash{\gGh}$ by selecting $g_e$s that are as ``simple'' as possible. 

As a corollary, we show that for a fixed total number of samples in $\smash{\gTh}$, the sample complexity will decrease as the number of environments increases. To see this, suppose that we have two situations corresponding to data generated respectively from $m$ and $m/b$ environments. The total sample complexity for each case will be respectively bounded by $\smash{O(\log \gC_\gGh(\frac{\eps_1}{16}, \gFh) + m\log \gC_\gFh(\frac{\eps_2}{16}, \gGh))}$ and $O(b\log \smash{\gC_{\gGh}}(\smash{\frac{\eps_1}{16}, \gFh)} + m\log \smash{\gC_\gFh(\frac{\eps_2}{16}, \gGh)})$. The latter being larger than the former, a situation with more environments  presents a clear advantage. \Figref{fig:plot} in \Secref{sec:exp} confirms this result with empirical evidence.

\vspace{-6pt}
\paragraph{\LEADS sample complexity for novel environments.} 

Suppose that problem \Eqref{eq:opt} has been solved for a set of environments  $E$, can we use the learned model for a new environment not present in the initial training set (\RQtwo)? Let $e'$ be such a new environment, $\gP_{e'}$ the trajectory distribution of $e'$, generated from dynamics $f_{e'} \sim Q$, and $\smash{\gPh_{e'}}$ an associated training set of size $n'$. The following results show that the number of required examples for reaching a given performance is much lower when training $f+g_{e'}$ with $f$ fixed on this new environment than training another $f'+g_{e'}$ from scratch (proof see \Suppref{supp:sample_complexity_general}).
\begin{restatable}{prop}{propthree}
\label{prop:sample_complexity_after_training}
For all $\eps, \delta$ with $0 < \eps, \delta <1$ if the number of samples $n'$ satisfies
\begin{equation}
    {n'\geq \max\bigg\{\frac{64}{\eps^2}\log\frac{4\gC(\frac{\eps}{16}, f + \smash{\gGh})}{\delta}, \frac{16}{\eps^2}\bigg\}},
\end{equation}then with probability at least $1-\delta$ (over the choice of novel training set $\smash{\gPh_{e'}}$), any learner $f+g_{e'} \in f+\smash{\gGh}$ will satisfy 
\(\smash{\er_{\gP_{e'}}(f+g_{e'}) \leq \hat{\er}_{\gPh_{e'}}(f+g_{e'}) + \eps}\).
\end{restatable}
In Prop.~\ref{prop:sample_complexity_after_training} as the capacity of $\smash{\gFh}$ no longer appears, the number of required samples now depends  only on the capacity of $\smash{f+\gGh}$. This sample complexity is then smaller than learning from scratch $f_{e'}= f+g_{e'}$ as can be seen by comparing with Prop.~\ref{prop:sample_complexity_all_envs} at $m=1$.

From the previous propositions, it is clear that the environment-specific functions $g_e$ need to be explicitly controlled. We now introduce a practical way to do that.
Let $\omega(r, \eps)$ be a strictly increasing function \wrt $r$ such that 
\begin{equation}\textstyle
    \log\gC_\gFh(\eps, \gGh) \leq \omega(r, \eps), ~~~ r = \sup_{g\in \gGh}\Omega(g)
\end{equation}
Minimizing $\Omega$ would reduce $r$ and then the sample complexity of our model by constraining $\smash{\gGh}$. Our goal is thus to construct such a pair $(\omega, \Omega)$. In the following, we will first show in \Secref{sec:linear-case}, how one can construct a penalization term $\Omega$ based on the covering number bound for linear approximators and linear ODEs. We show with a simple use case that the generalization error obtained in practice follows the same trend as the theoretical error bound when the number of environments varies. Inspired by this result, we then propose in \Secref{subsec:nn_instantiation} an effective $\Omega$ to penalize the complexity of the neural networks $g_e$.

\vspace{-6pt}
\subsection{Linear case: theoretical bounds correctly predict the trend of test error \label{sec:linear-case}}
\vspace{-6pt}
Results in \Secref{subsec:general_case} provide general guidelines for our approach. We now apply them to a linear system to see how the empirical results meet the tendency predicted by theoretical bound.

Let us consider a linear ODE ${\frac{\diff x^e_t}{\diff t}}\!=\!{\Lm_{\mF_e}(x^e_t)}$
where ${\Lm_{\mF_e}\!:\!x\!\mapsto\!\mF_e x}$ is a linear transformation associated to the square real valued matrix $\smash{\mF_e\in M_{d,d}(\R)}$. We choose as hypothesis space the space of linear functions $\smash{\gFh}\!\subset\!\gL(\R^d, \R^d)$ and instantiate a linear \LEADS \(
    {\frac{\diff x^e_t}{\diff t}\!=\!(\Lm_{\mF}+\Lm_{\mG_e})(x^e_t)}
\), $\smash{\Lm_{\mF}}\!\in \smash{\gFh}, \smash{\Lm_{\mG_e}}\!\in\smash{\gGh}\subseteq\smash{\gFh}$. As suggested in \cite{Bartlett2017}, we have that (proof in \Suppref{supp:bound-linear}):
\begin{restatable}{prop}{propfour}
\label{prop:sample_complexity_linear}
If for all linear maps $\smash{\Lm_{\mG_e}\!\in \gGh}$, $\|\mG\|_F^2 \leq\!r$, if the input space is bounded s.t.\ $\|x\|_2 \leq\!b$, and the MSE loss function is bounded by $c$, then 
\[
    \smash{\log \gC_{\gFh}(\varepsilon, {\gGh})}\leq
 \smash{\lceil\sfrac{rcd(2b)^2}{\eps^2}\rceil}\log2d^2 =: \omega(r,\eps)
\]
\end{restatable}
$\omega(r,\eps)$ is a strictly increasing function \wrt $r$. This indicates that we can choose $\smash{\Omega(\Lm_\mG)\!=\!\|\mG\|_F}$ as our optimization objective in \Eqref{eq:opt}. The sample complexity in \Eqref{eq:known-envs} will decrease with the size the largest possible $r\!=\!\smash{\sup_{\Lm_\mG\in\gGh}\Omega(\Lm_\mG)}$. The optimization process will reduce $\smash{\Omega(\Lm_\mG)}$ until a minimum is reached. The maximum size of the effective hypothesis space is then bounded and decreases throughout training thanks to the penalty. Then in linear case Prop.~2 becomes (proof \cf \Suppref{supp:bound-linear}):
\begin{restatable}{prop}{propfive}
If for linear maps  $\smash{\Lm_\mF \in \gFh}$, $\smash{\|\mF\|^2_F\leq\!r'}$, $\smash{\Lm_\mG\in\gGh}$, $\smash{\|\mG\|^2_F\leq\!r}$, $\smash{\|x\|_2\leq\!b}$, and if the MSE loss function is bounded by $c$, given $m$ environments and $n$ samples per environment, with the probability $1-\delta$, the generalization error upper bound is
\(\eps=\max\big\{\sqrt{\sfrac{(p+\sqrt{p^2+4q})}{2}}, \sqrt{\sfrac{16}{n}}\big\}\)
where ${p = \frac{64}{mn}\log\frac{4}{\delta}}$ and ${q = \frac{64}{n}\big\lceil\big(\frac{r'}{mz^2}+\frac{r}{(1-z)^2}\big)  cd(32b)^2\big\rceil\log2d^2}$ for any $0<z<1$.
\end{restatable}
\begin{wrapfigure}{r}{0.55\textwidth}
\vspace{-10pt}
    \centering
    \includegraphics[width=\linewidth]{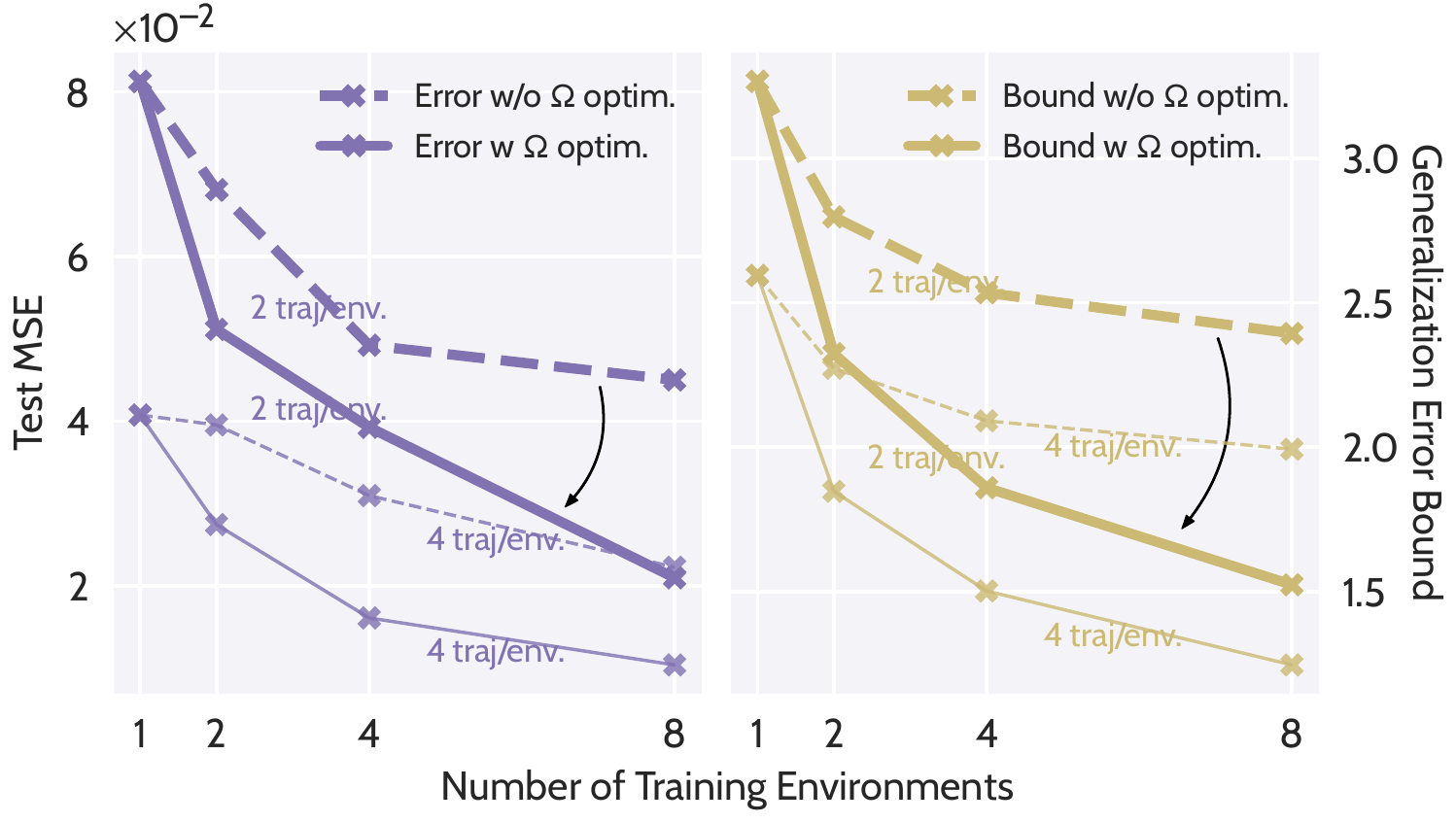}
    \vspace{-15pt}
    \caption{Test error compared with corresponding theoretical bound. The arrows indicate the changes after applying $\Omega(g_e)$ penalty.}
    \label{fig:linear_consistency}
    \vspace{-13pt}
\end{wrapfigure}
In \Figref{fig:linear_consistency}, we take an instance of linear ODE defined by $\mF_e\!=\!\mQ\mLambda_e\mQ^\top$ with the diagonal $\mLambda_e$ specific to each environment After solving \Eqref{eq:opt} we have at the optimum that $\mG_e\!=\!\mF_e-\mF^\star\!=\!{\mF_e -\frac{1}{m}\sum_{e'\in E}{\mF_{e'}}}$. Then we can take $\smash{r = \max_{\{\Lm_{\mG_e}\}}\Omega(\Lm_{\mG_e})}$ as the norm bound of $\smash{\gGh}$ when $\Omega(g_e)$ is optimized. \Figref{fig:linear_consistency} shows on the left the test error with and without penalty and the corresponding theoretical bound on the right. We observe that, after applying the penalty $\Omega$, the test error is reduced as well as the theoretical generalization bound, as indicated by the arrows from the dashed line to the concrete one. 
See \Suppref{supp:bound-linear} for more details on this experiment.

\vspace{-6pt}
\subsection{Nonlinear case: instantiation for neural nets \label{subsec:nn_instantiation}}
\vspace{-6pt}
\camerareadychange{The above linear case validates the ideas introduced in Prop.~\ref{prop:sample_complexity_all_envs} and provides an instantiation guide and an intuition on the more complex nonlinear case. This motivates us to instantiate the general case by choosing an appropriate approximating space $\smash{\gFh}$ and a penalization function $\Omega$ from the generalization bounds for the corresponding space. \Suppref{supp:sample_complexity_nn} of the Appendix contains additional details justifying those choices.} For $\smash{\gFh}$, we select the space of feed-forward neural networks with a fixed architecture. We choose the following penalty function:\begin{equation}
   \Omega(g_e) = \|g_e\|_\infty^2 + \alpha \|g_e\|^2_{\mathrm{Lip}}\label{eq:omega_nn}
\end{equation}
where $\|g\|_\infty = \ess \sup |g|$ and $\|\cdot\|_{\mathrm{Lip}}$ is the Lipschitz semi-norm, $\alpha$ is a hyperparameter. This is inspired by the existing capacity bound for NNs~\cite{Haussler1992} (see \Suppref{supp:sample_complexity_nn} for details). 
Note that constructing tight generalization bounds for neural networks is still an open research problem~\cite{Nagarajan19}; however, it may still yield valuable intuitions and guide algorithm design.
This heuristic is tested successfully on three different datasets with different architectures in the experiments~(\Secref{sec:exp}).

\vspace{-6pt}
\section{Experiments\label{sec:exp}}
\vspace{-6pt}
Our experiments are conducted on three families of dynamical systems described by three broad classes of differential equations. All exhibit complex and nonlinear dynamics. The first one is an ODE-driven system used for biological system modeling. The second one is a PDE-driven reaction-diffusion model, well-known in chemistry for its variety of spatiotemporal patterns. The third one is the more physically complex Navier-Stokes equation, expressing the physical laws of incompressible Newtonian fluids. To show the general validity of our framework, we will use 3 different NN architectures (MLP, ConvNet, and Fourier Neural Operator~\cite{LiKALBSA2021}). Each architecture is well-adapted to the corresponding dynamics. This also shows that the framework is valid for a variety of approximating functions.

\vspace{-6pt}
\subsection{Dynamics, environments, and datasets}
\vspace{-6pt}
\paragraph{Lotka-Volterra (LV).} 
This classical model \cite{Lotka1926} is used for describing the dynamics of interaction between a predator and a prey. The dynamics follow the ODE:
\[
    \nicefrac{\diff u}{\diff t} = \alpha u- \beta uv,  \nicefrac{\diff v}{\diff t} = \delta uv - \gamma v
\]
with $u, v$ the number of prey and predator, $\alpha,\beta,\gamma,\delta\!>\!0$ defining how the two species interact. The system state is $x^e_t = (u^e_t, v^e_t)\in \smash{\R_+^2}$. The initial conditions $u_0^i, v_0^i$ are sampled from a uniform distribution $P_0$. We characterize the dynamics by $\theta=(\sfrac{\alpha}{\beta},\sfrac{\gamma}{\delta}) \in \Theta$. An environment $e$ is then defined by parameters $\theta_e$ sampled from a uniform distribution over a parameter set $\Theta$.
We then sample two sets of environment parameters: one used as training environments for \RQone, the other treated as novel environments. for \RQtwo.
\vspace{-6pt}
\paragraph{Gray-Scott (GS).}
This reaction-diffusion model is famous for its complex spatiotemporal behavior given its simple equation formulation \cite{Pearson1993}. The governing PDE is:
\begin{equation*}
    \nicefrac{\partial u}{\partial t} = D_u\Delta u - uv^2 + F(1-u), \nicefrac{\partial v}{\partial t} = D_v\Delta v + uv^2 - (F+k)v
\end{equation*}
where the $u, v$ represent the concentrations of two chemical components in the spatial domain $S$ with periodic boundary conditions, the spatially discretized state at time $t$ is $\smash{x^e_t = (u^e_t, v^e_t)\in \R_+^{2\times 32^2}}$. $D_u, D_v$ denote the diffusion coefficients respectively for $u, v$, and are held constant, and $F, k$ are the reaction parameters determining the spatio-temporal patterns of the dynamics \cite{Pearson1993}. As for the initial conditions $(u_0, v_0)\sim P_0$, we consider uniform concentrations, with 3 2-by-2 squares fixed at other concentration values and positioned at uniformly sampled positions in $S$ to trigger the reactions. 
An environment  $e$ is defined by its parameters $\theta_e=(F_e, k_e)\in \Theta$. We consider a set of $\theta_e$ parameters uniformly sampled from the environment distribution $Q$ on $\Theta$.
\vspace{-6pt}
\paragraph{Navier-Stokes (NS).} We consider the Navier-Stokes PDE for incompressible flows:
\[
\nicefrac{\partial w}{\partial t} = - v\cdot\nabla w + \nu\Delta w + \xi \qquad \nabla\cdot v = 0
\]
where $v$ is the velocity field, $w=\nabla \times v$ is the vorticity, both $v, w$ lie in a spatial domain $S$ with periodic boundary conditions, $\nu$ is the viscosity and $\xi$ is the constant forcing term in the domain $S$. The discretized state at time $t$ is the vorticity $x^e_t=w^e_t \in \smash{\R^{32^2}}$. Note that $v$ is already contained in $w$. 
We fix $\nu=10^{-3}$ across the environments. We sample the initial conditions $w^e_0\sim P_0$ as in \cite{LiKALBSA2021}. An environment $e$ is defined by its forcing term $\xi_e\in \Theta_\xi$. We uniformly sampled a set of forcing terms from $Q$ on $\Theta_\xi$.
\vspace{-6pt}
\paragraph{Datasets.}
For training, we create two datasets for LV by simulating trajectories of $K\!=\!20$ successive points with temporal resolution $\Delta t\!=\!0.5$. We use the first one as a set of training dynamics to validate the \LEADS framework. We choose 10 environments  and simulate 8 trajectories (thus corresponding to $n\!=\!8\!\cdot\!K$ data points) per environment  for training. We can then easily control the number of data points and environments in experiments by taking different subsets. The second one is used to validate the improvement with \LEADS while training on novel environments. We simulate 1 trajectory ($n\!=\!1\!\cdot\!K$ data points) for training. We create two datasets for further validation of \LEADS with GS and NS. For GS, we  simulate trajectories of $K\!=\!10$ steps with $\Delta t\!=\!40$. We choose 3 parameters and simulate 1 trajectory ($n\!=\!1\!\cdot\!K$ data points) for training. For NS, we simulate trajectories of $K\!=\!10$ steps with $\Delta t\!=\!1$. We choose 4 forcing terms and simulate 8 trajectories ($n\!=\!8\!\cdot\!K$ states) for training. For test-time evaluation, we create for each equation in each environment a test set of 32 trajectories ($32\!\cdot\!K$) data points. Note that every environment dataset has the same number of trajectories and the initial conditions are fixed to equal values across the environments to ensure that the data variations only come from the dynamics themselves, \ie for the $i$-th trajectory in $\smash{\gPh_e}$, $\smash{\forall e, x^{e,i}_0\!=\!x^i_0}$. LV and GS data are simulated with the DOPRI5 solver in NumPy \cite{DormandP1980,Harris2020}. NS data is simulated with the pseudo-spectral method as in \cite{LiKALBSA2021}.

\begin{figure}[t]
    \vspace{-13pt}
    \centering
    \fontfamily{Cabin-TLF}\selectfont
    \subfloat{\tiny%
    \setlength{\tabcolsep}{0.8pt}%
    \begin{tabular}{ccc}
        \OPE & FT-NODE & \textbf{\LEADS} \\
         \includegraphics[width=0.07\textwidth]{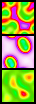}\hspace{-1.3
         pt}\includegraphics[width=0.07\textwidth]{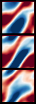} & \includegraphics[width=0.07\textwidth]{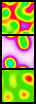}\hspace{-1.3pt}\includegraphics[width=0.07\textwidth]{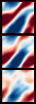} & \includegraphics[width=0.07\textwidth]{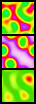}\hspace{-1.3pt}\includegraphics[width=0.07\textwidth]{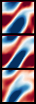}\\
         GS\hspace*{20pt}NS & GS\hspace*{20pt}NS & GS\hspace*{20pt}NS 
    \end{tabular}}
    \hfill{\color{gray}\vrule width 0.7pt height 50pt}\hfill
    \subfloat{\tiny%
    \setlength{\tabcolsep}{0.8pt}%
    \begin{tabular}{c}
         Ground truth \\
         \includegraphics[width=0.07\textwidth]{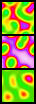}\hspace{-1.3pt}\includegraphics[width=0.07\textwidth]{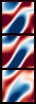}\\
         GS\hspace*{20pt}NS
    \end{tabular}}
\hfill{\color{gray}\vrule width 0.7pt height 50pt}\hfill
\subfloat{\tiny%
    \setlength{\tabcolsep}{0.8pt}%
    \begin{tabular}{ccc}
    \OPE & FT-NODE & \textbf{\LEADS} \\
    \includegraphics[width=0.07\textwidth]{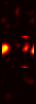}\hspace{-1.3pt}\includegraphics[width=0.07\textwidth]{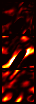} &
    \includegraphics[width=0.07\textwidth]{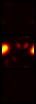}\hspace{-1.3pt}\includegraphics[width=0.07\textwidth]{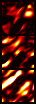} &
    \includegraphics[width=0.07\textwidth]{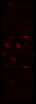}\hspace{-1.3pt}\includegraphics[width=0.07\textwidth]{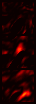}\\
    GS\hspace*{20pt}NS & GS\hspace*{20pt}NS & GS\hspace*{20pt}NS
    \end{tabular}}
    \vspace{-5pt}
    \caption{Left: final states for GS and NS predicted by the two best baselines (\OPE and FT-NODE) and \LEADS compared with ground truth. Different environment are arranged by row (3 in total). Right: the corresponding MAE error maps, \camerareadychange{the scale of the error map is [0, 0.6] for GS, and [0, 0.2] for NS}; darker is smaller. (See \Suppref{supp:extra-exp} for full sequences)}
    \label{fig:result-gs-comp-pred}
    \vspace{-10pt}
\end{figure}

\begin{figure}[t]
    \centering
    \includegraphics[width=\textwidth]{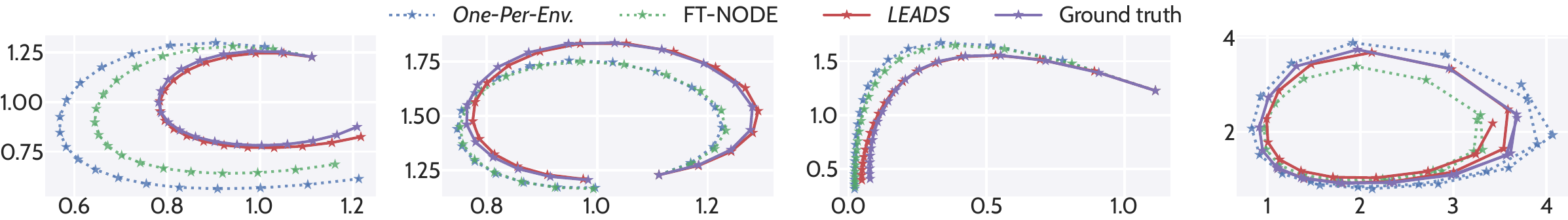}
    \caption{Test predicted trajectories in phase space with two baselines (\OPE and FT-NODE) and \LEADS compared with ground truth for LV for 4 envs., one per figure from left to right. Quantity of the prey $u$ and the predator $v$ respectively on the horizontal and the vertical axis. Initial state is the rightmost end-point of the figures and it is common to all the trajectories.} 
    \label{fig:result-lv-comp-pred}
    \vspace{-10pt}
\end{figure}

\vspace{-6pt}
\subsection{Experimental settings and baselines}
\vspace{-6pt}
We validate \LEADS in two settings: in the first one all the environments  in $E$ are available at once and then $f$ and all the $g_e$s are all trained on $E$. In the second one, training has been performed on $E$ as before, and we consider a novel environment  ${e'}\not\in E$: the shared term $f$ being kept fixed, the approximating function $f_{e'}=f+g_{e'}$ is trained on the data from ${e'}$ (\ie only $g_{e'}$ is modified).
\vspace{-6pt}
\paragraph{All environments available at once.} 
We introduce five baselines used for comparing with \LEADS:
\begin{enumerate*}[label=(\alph*)]
    \item \textbf{\OFA}: learning on the entire dataset $\smash{\gPh}$ over all environments  with the sum of a pair of NNs $f+g$, with the standard ERM principle, as in \cite{AyedDPBG2019}. Although this is equivalent to use only one function $f$, we use this formulation to indicate that the number of parameters is the same for this experiment and for the \LEADS ones.
    \item \textbf{\OPE}: learning a specific function for each dataset $\smash{\gPh_e}$. For the same reason as above, we keep the sum formulation $(f+g)_e$. 
    \item Factored Tensor RNN or \textbf{FT-RNN }\cite{SpieckermannDUHR2015}: it modifies the recurrent neural network to integrate a one-hot environment code into each linear transformation of the network. Instead of being encoded in a separate function $g_e$ like in \LEADS, the environment appears here as an extra one-hot input for the RNN linear transformations. This can be implemented for representative SOTA (spatio-)temporal predictors such as GRU~\cite{ChoVGBBSB2014} or PredRNN~\cite{WangLWGY2017}.
    \item \textbf{FT-NODE}: a baseline for which the same environment encoding as FT-RNN is incorporated in a Neural ODE \cite{ChenRBD2018}.
    \item \camerareadychange{Gradient-based Meta Learning or \textbf{GBML-like} method: we propose a GBML-like baseline which can directly compare to our framework. It follows the principle of MAML \cite{FinnAL2017}, by training \OFA at first which provides an initialization near to the given environments like GBML does, then fitting it individually for each training environment.} 
    \item \textbf{\LEADSNM}: ablation baseline, our proposal without the $\Omega(g_e)$ penalization. 
\end{enumerate*}
A comparison with the different baselines is proposed in \Tabref{tab:comparison} for the three dynamics. For concision, we provide a selection of results corresponding to 1 training trajectory per environment for LV and GS and 8 for NS. This is the minimal training set size for each dataset. Further experimental results when varying the number of environments from 1 to 8 are provided in \Figref{fig:plot} and \Tabref{tab:grid} for LV.

\vspace{-6pt}
\paragraph{Learning on novel environments.} We consider the following training schemes with a pre-trained, fixed $f$:
\begin{enumerate*}[label=(\alph*)]
    \item \textbf{\PRT}: only the pre-trained $f$ is used for prediction; a sanity check to ensure that $f$ cannot predict in any novel environment without further adaptation.
    \item \textbf{\OPE}: training from scratch on $\smash{\{\gPh_{e'}\}}$ as \OPE in the previous section. 
    \item \textbf{\LEADSBS}: we train $g$ on each dataset $\smash{\gPh_{e'}}$ based on pre-trained $f$, \ie $f+g_{e'}$, leaving only $g_{e'}$s adjustable.
\end{enumerate*} We compare the test error evolution during training for 3 schemes above for a comparison of convergence speed and performance. Results are given in \Figref{fig:novel_env}.
\vspace{-6pt}
\subsection{Experimental results}
\label{subsec:experimental-results}

\begin{table}[t]
\vspace{-14pt}
    \centering
    \setlength{\tabcolsep}{9pt}
    \caption{Results for LV, GS, and NS datasets, trained on $m$ envs.\ with $n$ data points per env.}
    \resizebox{\textwidth}{!}{
    \begin{tabular}{lcccccc}
    \toprule
    \multirow{2}{*}{Method} & \multicolumn{2}{c}{LV ($m=10, n=1\cdot K$)} & \multicolumn{2}{c}{GS ($m=3, n=1\cdot K$)} & \multicolumn{2}{c}{NS ($m=4, n=8\cdot K$)} \\
    \cmidrule(lr){2-3}\cmidrule(lr){4-5}\cmidrule(lr){6-7}
    & MSE train & MSE test & MSE train & MSE test & MSE train & MSE test\\ 
    \midrule
    \OFA & 4.57e-1 & 5.08$\pm$0.56 e-1 & 1.55e-2 & 1.43$\pm$0.15 e-2 & 5.17e-2 & 7.31$\pm$5.29 e-2 \\ 
    \OPE & 2.15e-5 & 7.95$\pm$6.96 e-3 & 8.48e-5 & 6.43$\pm$3.42 e-3 & 5.60e-6 & 1.10$\pm$0.72 e-2 \\
    FT-RNN \cite{SpieckermannDUHR2015} & 5.29e-5 & 6.40$\pm$5.69 e-3 & 8.44e-6 & 8.19$\pm$3.09 e-3 & 7.40e-4 & 5.92$\pm$4.00 e-2 \\
    FT-NODE & 7.74e-5 & 3.40$\pm$2.64 e-3 & 3.51e-5 & 3.86$\pm$3.36 e-3 & 1.80e-4 & 2.96$\pm$1.99 e-2\\
    GBML-like & 3.84e-6 & 5.87$\pm$5.65 e-3 & 1.07e-4 & 6.01$\pm$3.62 e-3 & 1.39e-4 & 7.37$\pm$4.80 e-3 \\
    \LEADSNM & 3.28e-6 & 3.07$\pm$2.58 e-3 & 7.65e-5 & 5.53$\pm$3.43 e-3 & 3.20e-4 & 7.10$\pm$4.24 e-3 \\
    \textbf{\LEADS} (Ours) & 5.74e-6 & \textbf{1.16$\pm$0.99 e-3} & 5.75e-5 & \textbf{2.08$\pm$2.88 e-3} & 1.03e-4 & \textbf{5.95$\pm$3.65 e-3}\\
    \bottomrule
    \end{tabular}
    }
    \label{tab:comparison}
\vspace{-10pt}
\end{table}
\vspace{-6pt}

\paragraph{All environments available at once.}
We show the results in \Tabref{tab:comparison}. For LV systems, we confirm first that the entire dataset cannot be learned properly with a single model (\OFA) when the number of environments  increases. Comparing with other baselines, our method \LEADS reduces the test MSE over 85\% \wrt \OPE and over 60\% \wrt \LEADSNM, we also cut 50\%-75\% of error \wrt other baselines. \Figref{fig:result-lv-comp-pred} shows samples of predicted trajectories in test, \LEADS follows very closely the ground truth trajectory, while \OPE under-performs in most environments. We observe the same tendency for the GS and NS systems. \camerareadychange{The error is reduced by: around 2/3 (GS) and 45\% (NS) \wrt \OPE; over 60\% (GS) and 15\% (NS) \wrt \LEADSNM; 45-75\% (GS) and 15-90\% (NS) \wrt other baselines.} In \Figref{fig:result-gs-comp-pred}, the final states obtained with \LEADS are qualitatively closer to the ground truth. Looking at the error maps on the right, we see that the errors are systematically reduced across all environments compared to the baselines. This shows that \LEADS accumulates less errors through the integration, which suggests that \LEADS alleviates overfitting. 

\begin{wrapfloat}{figure}{r}{0.55\textwidth}
\centering
\vspace{-10pt}
    \includegraphics[width=\linewidth]{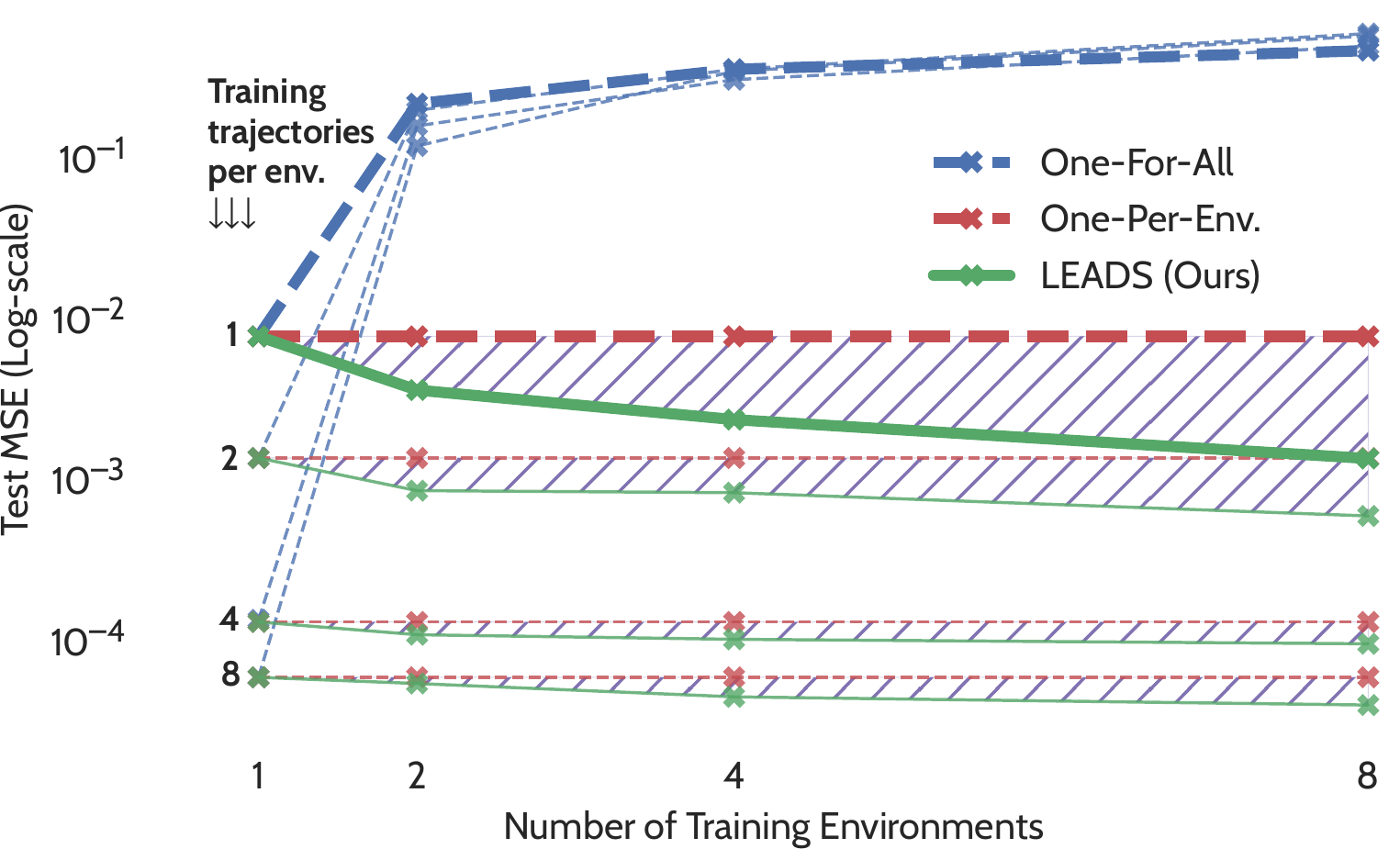}
    \caption{Test error for LV \wrt the number of environments. We apply the models in 1 to 8 environments. 4 groups of curves correspond to models trained with 1 to 8 trajectories per env. All groups highlight the same tendencies: increasing \OFA, stable \OPE, and decreasing \LEADS. More results of baselines methods in \Suppref{supp:extra-exp}.}
    \label{fig:plot}
\vspace{-14pt}
\end{wrapfloat}

We have also conducted a larger scale experiment on LV (\Figref{fig:plot}) to analyze the behavior of the different training approaches as the number of environments  increases. We consider three models \OFA, \OPE and \LEADS, 1, 2, 4 and 8 environments, and for each such case, we have 4 groups of curves, corresponding to 1, 2, 4 and 8 training trajectories per environment. We summarize the main observations. With \OFA (blue), the error increases as the number of environments increases: the dynamics for each environment  being indeed different, this introduces an increasingly large bias, and thus the data cannot be fit with one single model. The performance of \OPE (in red), for which models are trained independently for each environment, is constant as expected when the number of environments changes. \LEADS (green) circumvents these issues and shows that the shared characteristics among the environments can be leveraged so as to improve generalization: it is particularly effective when the number of samples per environment is small. 
 (See \Suppref{supp:extra-exp} for more details on the experiments and on the results).
 
\begin{wrapfloat}{figure}{r}{0.55\textwidth}
\centering
\vspace{-10pt}
    \includegraphics[width=\linewidth]{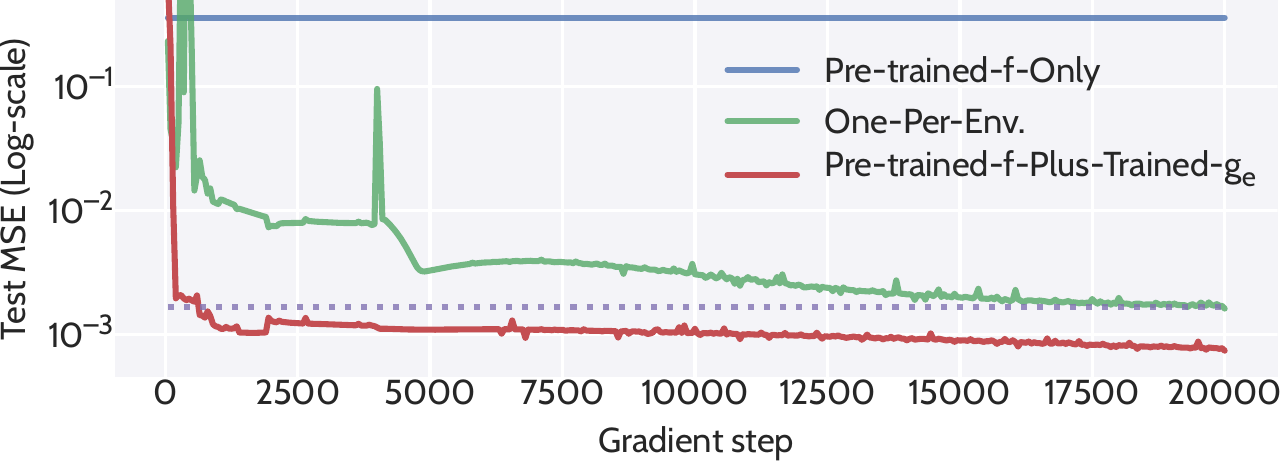}
    \caption{Test error evolution during training on 2 novel environments for LV.}
    \label{fig:novel_env}
\vspace{-10pt}
\end{wrapfloat}
 
\vspace{-6pt}
\paragraph{Learning on novel environments.}

We demonstrate how the pre-trained dynamics can help to fit a model for novel environments. We took an $f$ pre-trained by \LEADS on a set of LV environments. \Figref{fig:novel_env} shows the evolution of the test loss during training for three systems: a $f$ function pre-trained by \LEADS on a set of LV training environments, a $g_e$ function trained from scratch on the new environment and \LEADS that uses a pre-trained $f$ and learns a $g_e$ residue on this new environment. \PRT alone cannot predict in any novel environments. Very fast in the training stages, \LEADSBS already surpasses the best error of the model trained from scratch (indicated with dotted line). \camerareadychange{Similar results are also observed with the GS and NS datasets (\cf \Suppref{supp:extra-exp}, \Tabref{tab:train-new-env}).} These empirical results clearly show that the learned shared dynamics accelerates and improves the learning in novel environments.
\vspace{-6pt}
\subsection{Training and implementation details\label{subsec:training-details}}
\paragraph{Discussion on trajectory-based optimization.} Solving the learning problem \Eqref{eq:constraint} in our setting, involves computing a trajectory loss (integral term in \Eqref{eq:lagrangian}). However, in practice, we do not have access to the continuous trajectories at every instant $t$ but only to a finite number of snapshots for the state values $\smash{\{x_{k\Delta t}\}_{0 \leq k\leq \frac{T}{\Delta t}}}$ at a temporal resolution $\Delta t$. From these discrete observed trajectories, it is still possible to recover an approximate derivative  ${d^\Lambda_{k\Delta t} \simeq {\frac{\mathrm{d}x_{k\Delta t}}{\mathrm{d}t}}}$ using a numerical scheme $\Lambda$. The integral term for a given sample in the objective \Eqref{eq:lagrangian} would then be estimated as ${\sum_{k=1}^{K}\left\|d^\Lambda_{k\Delta t}  - (f + g_e)(x_{\Delta t k})\right\|^2}$. This is not the best solution and we have observed much better prediction performance for all models, including the baselines, when computing the error directly on the states, using an integral formulation ${{\sum_{k=1}^{K}\|x_{(k+1)\Delta t} -\tilde{x}_{(k+1)\Delta t}\|^2}}$, where $\tilde{x}_{(k+1)\Delta t}$ is the solution given by a numerical solver approximating the integral ${x_{k\Delta t} + \int_{k\Delta t}^{(k+1)\Delta t}\!\left(f + g_e\right)\!(\tilde{x}_s)\diff s}$ starting from $x_{k\Delta t}$. Comparing directly in the state space yields more accurate results for prediction as the learned network tends to correct the solver's numerical errors, as first highlighted in \cite{YinLDDATG2021}.
\vspace{-6pt}
\paragraph{Calculating $\Omega$.} Given finite data and time, the exact infinity norm and Lipschitz norm are both intractable. We opt for more practical forms in the experiments. For the infinity norm, we chose to minimize the empirical norm of the output vectors on known data points, this choice is motivated in \Suppref{supp:discussion}. In practice, we found out that dividing the output norm by its input norm works better:  ${\frac{1}{n}\sum_{i,k}\|g_e(x^{e,i}_{k\Delta t})\|^2 / \|x^{e,i}_{k\Delta t}\|^2}$, where the ${x^{e,i}_{k\Delta t}}$ are known states in the training set.
For the Lipschitz norm, as suggested in \cite{BiettiMCM2019}, we optimize the sum of the spectral norms of the weight at each layer ${\sum_{l=1}^D \|W^{g_e}_{l}\|^2}$. We use the power iteration method in \cite{MiyatoKKY2018} for fast spectral norm approximation.

\vspace{-6pt}
\paragraph{Implementation.} We used 4-layer MLPs for LV, 4-layer ConvNets for GS and Fourier Neural Operator (FNO)~\cite{LiKALBSA2021} for NS. For FT-RNN baseline, we adapted GRU \cite{ChoVGBBSB2014} for LV and PredRNN \cite{WangLWGY2017} for GS and NS. We apply the Swish function~\cite{RamachandranZL2017} as the default activation function. Networks are integrated in time with RK4 (LV, GS) or Euler (NS), using the basic back-propagation through the internals of the solver. We apply an exponential Scheduled Sampling \cite{LambGZZCB2016} with exponent of $0.99$ to stabilize the training. We use the Adam optimizer~\cite{KingmaB2014} with the same learning rate $\smash{10^{-3}}$ and $(\beta_1, \beta_2) = (0.9, 0.999)$ across the experiments. For the hyperparamters in \Eqref{eq:omega_nn}, we chose respectively $\lambda = 5\times10^3, 10^2,10^5$ and $\alpha = 10^{-3}, 10^{-2}, 10^{-5}$ for LV, GS and NS. 
All experiments are performed with a single NVIDIA Titan Xp GPU.

\vspace{-6pt}
\section{Related work}
\vspace{-6pt}
Recent approaches linking invariances to Out-of-Distribution (OoD) Generalization, such as \cite{ArjovskyBGL2020,KruegerCJZBLC2020,Teney2020}, aim at finding a single classifier that predicts well invariantly across environments with the power of extrapolating outside the known distributions. However, in our dynamical systems context, the optimal regression function should be different in each environment, and modeling environment bias is as important as modeling the invariant information, as both are indispensable for prediction. Thus such invariant learners are incompatible with our setting. Meta-learning methods have recently been considered for dynamical systems as in \cite{FinnAL2017,LeeYS2021}. Their objective is to train a single model that can be quickly adapted to a novel environment with a few data-points in limited training steps. \camerareadychange{However, in general these methods do not focus on leveraging the commonalities and discrepencies in data and may suffer from overfitting at test time \cite{MishraRCA2017}.}
Multi-task learning~\cite{Zhang2018} seeks for learning shared representations of inputs that exploit the domain information. Up to our knowledge current multi-task methods have not been considered for dynamical systems. \cite{SpieckermannDUHR2015} apply multi-task learning for interactive physical environments but do not consider the case of dynamical systems. Other approaches like \cite{Yildiz2019,NorcliffeBDML2021} integrate probabilistic methods into a Neural ODE, to learn a distribution of the underlying physical processes. Their focus is on the uncertainty of a single
system. \cite{YinLDDATG2021} consider an additive decomposition but focus on the combination of physical and statistical components for a single process and not on learning from different environments.

\vspace{-6pt}
\section{Discussions\label{sec:discussions}}
\vspace{-6pt}
\paragraph{Limitations} Our framework is generic and could be used in many different contexts. On the theoretical side, the existence and uniqueness properties (Prop.~\ref{prop:exist_unique}) rely on relatively mild conditions covering a large number of situations. The complexity analysis, on the other side, is only practically relevant for simple hypothesis spaces (here linear), and then serves for developing the intuition on more complex spaces (NNs here) where bounds are too loose to be informative.
Another limitation is that the theory and experiments consider deterministic systems only: the experimental validation is performed on simulated deterministic data. Note however that this is the case in the vast majority of the ML literature on ODE/PDE spatio-temporal modeling  \cite{RaissiPK2019,LongLD2018,LiKALBSA2021,YinLDDATG2021}.
In addition, modeling complex dynamics from real world data is a problem by itself.

\vspace{-6pt}
\paragraph{Conclusion}
We introduce \LEADS, a data-driven framework to learn dynamics from data collected from a set of distinct dynamical systems with commonalities. Experimentally validated with three families of equations, our framework can significantly improve the test performance in every environment \wrt classical training, especially when the number of available trajectories is limited. We further show that the  dynamics extracted by \LEADS can boost the learning in similar new environments, which gives us a flexible framework for generalization in novel environments. More generally, we believe that this method is a promising step towards addressing the generalization problem for learning dynamical systems and has the potential to be applied to a large variety of problems.

\section*{Acknowledgements}
We acknowledge financial support from the ANR AI Chairs program DL4CLIM ANR-19-CHIA-0018-01.

\bibliography{bib.bib}
\bibliographystyle{abbrv}

\section*{Checklist}
\begin{enumerate}
\item For all authors...
\begin{enumerate}
  \item Do the main claims made in the abstract and introduction accurately reflect the paper's contributions and scope?
    \answerYes
  \item Did you describe the limitations of your work?
    \answerYes In \Secref{sec:discussions}.
  \item Did you discuss any potential negative societal impacts of your work?
    \answerNo
    
    The only relevant societal impact is around the computational cost, while we use very limited computation power (maximum with a single GPU). 
  \item Have you read the ethics review guidelines and ensured that your paper conforms to them?
    \answerYes
\end{enumerate}

\item If you are including theoretical results...
\begin{enumerate}
  \item Did you state the full set of assumptions of all theoretical results? 
    \answerYes 
    
    In \Secref{sec:setting} and \Secref{sec:sample_complexity}.
	\item Did you include complete proofs of all theoretical results?
    \answerYes 
    
    In the supplemental material.
\end{enumerate}

\item If you ran experiments...
\begin{enumerate}
  \item Did you include the code, data, and instructions needed to reproduce the main experimental results (either in the supplemental material or as a URL)? \answerYes 
  
  We will provide the code in the supplemental material.
  \item Did you specify all the training details (e.g., data splits, hyperparameters, how they were chosen)? \answerYes In \Secref{sec:exp}.
	\item Did you report error bars (e.g., with respect to the random seed after running experiments multiple times)? \answerYes In \Secref{sec:exp}.
	\item Did you include the total amount of compute and the type of resources used (e.g., type of GPUs, internal cluster, or cloud provider)? \answerYes In \Secref{sec:exp}.
\end{enumerate}

\item If you are using existing assets (e.g., code, data, models) or curating/releasing new assets...
\begin{enumerate}
  \item If your work uses existing assets, did you cite the creators?
    \answerYes In \Secref{sec:exp}.
  \item Did you mention the license of the assets?
    \answerYes In the supplemental material.
  \item Did you include any new assets either in the supplemental material or as a URL?
    \answerYes 
  \item Did you discuss whether and how consent was obtained from people whose data you're using/curating?
    \answerNA 
  \item Did you discuss whether the data you are using/curating contains personally identifiable information or offensive content?
    \answerNA
\end{enumerate}

\item If you used crowdsourcing or conducted research with human subjects...
\begin{enumerate}
  \item Did you include the full text of instructions given to participants and screenshots, if applicable?
    \answerNA
  \item Did you describe any potential participant risks, with links to Institutional Review Board (IRB) approvals, if applicable?
    \answerNA
  \item Did you include the estimated hourly wage paid to participants and the total amount spent on participant compensation?
    \answerNA
\end{enumerate}

\end{enumerate}

\clearpage
\appendix
\setcounter{page}{1}

\title{\textit{LEADS}: Learning Dynamical\\ Systems that Generalize Across Environments\\ \normalfont{Supplemental Material}}

\maketitle

\setcounter{prop}{0}
\setcounter{theorem}{0}
\setcounter{equation}{0}
\setcounter{table}{0}
\setcounter{figure}{0}
\setcounter{cor}{0}
\setcounter{lemma}{0}

\newtheorem{manualpropositioninner}{Proposition}
\newenvironment{manualproposition}[1]{%
  \renewcommand\themanualpropositioninner{#1}%
  \manualpropositioninner
}{\endmanualpropositioninner}

\renewcommand{\theequation}{S\arabic{equation}}
\renewcommand{\thetable}{S\arabic{table}}
\renewcommand{\thefigure}{S\arabic{figure}}
\renewcommand{\thetheorem}{S\arabic{theorem}}
\renewcommand{\theprop}{S\arabic{prop}}
\renewcommand{\thecor}{S\arabic{cor}}
\renewcommand{\thelemma}{S\arabic{lemma}}

\newenvironment{subproof}[1][\proofname]{%
  \renewcommand{\qedsymbol}{$\blacksquare$}%
  \proof[#1]}
  {\endproof}%


\section{\label{supp:proof}Proof of Proposition~\ref{prop:exist_unique}}
\propone*
\begin{proof}
The optimization problem is:
\begin{equation}
\min\limits_{f,g_e\in\gF} ~~~ \sum_{e\in E}\Omega(g_e) ~~~
\mathrm{subject~to} ~~~
\forall x^{e,i} \in\gTh, \forall t,  \frac{\mathrm{d}x_t^{e,i}}{\mathrm{d} t} =(f+g_e)(x_t^{e,i})\label{eq:supp_opt_sup} \tag{3}
\end{equation}

The idea is to first reconstruct the full functional from the trajectories of $\gTh$. By definition, $\gA^e$ is the set of points reached by trajectories in $\gTh$ from environment $e$ so that:
\[
\gA^e = \{x\in\R^d\ |\ \exists x^e_\cdot\in\gTh, \exists t,\ x_t^e = x\}
\]
Then let us define a function $f^{\mathrm{data}}_e$ in the following way, $\forall e \in E$, take $a\in \gA^e$, we can find $x^e_\cdot\in\gTh$ and $t_0$ such that $x_{t_0}^e = a$. Differentiating $x^e_\cdot$ at $t_0$, which is possible by definition of $\gTh$, we take:
\[
f^{\mathrm{data}}_e(a) = \left.\frac{\diff x_t^e}{\diff t}\right|_{t=t_0}
\]
For any $(f,g_e)$ satisfying the constraint in \Eqref{eq:supp_opt_sup}, we then have  $(f+g_e)(a) = \left.\frac{\diff x_t}{\diff t}\right|_{t_0} =f^{\mathrm{data}}_e(a)$ for all $a\in\gA^e$. Conversely, any pair such that $(f, g_e)\in\gF\times\gF$ and ${f+g_e =f^{\mathrm{data}}_e}$, verifies the constraint.

Thus we have the equivalence between \Eqref{eq:supp_opt_sup} and the following objective:
\begin{gather}
\label{eq:opt_sup}
\min\limits_{f\in\gF} ~~~ \sum_e\Omega(f^{\mathrm{data}}_e - f)
\end{gather}
The result directly follows from the fact that the objective is a sum of (strictly) convex functions in $f$ and is thus (strictly) convex in $f$.
\end{proof}

\section{Further details on the generalization with \textit{LEADS}  \label{supp:sample_complexity_details}}
In this section, we will give more details on the link between our framework and its generalization performance. After introducing the necessary definitions in \Secref{supp:sample_complexity_pre}, we show the proofs of the results for the general case in \Secref{sec:sample_complexity}. Then in \Secref{supp:bound-linear} we provide the instantiation for linear approximators. Finally, we show how we derived our heuristic instantiation for neural networks in \Eqref{eq:omega_nn} in \Secref{subsec:nn_instantiation} from the existing capacity bound for neural networks.
 
\subsection{Preliminaries \label{supp:sample_complexity_pre}}

\Tabref{tab:def} gives the definition of the different
capacity instances considered in the paper for each hypothesis space, and the associated distances.
We say that a space $\gH$ is $\eps$-covered by a set $H$, with respect to a metric or pseudo-metric $d(\cdot, \cdot)$, if for all $h\in\gH$ there exists $h' \in H$ with $d(h,h')\leq \eps$. We define by $\gN(\eps, \gH, d)$ the cardinality of the smallest $H$ that $\eps$-covers $\gH$, also called covering number
\citeS{shalev-shwartz_ben-david_2014S}. The capacity of each hypothesis space is then defined by the maximum covering number over all distributions. Note that the loss function is involved in every metric in \Tabref{tab:def}. For simplicity, we therefore omit the notation of loss function for the hypothesis spaces.

As in \citeS{Baxter2000S}, covering numbers are based on pseudo-metrics. We can verify that all distances in \Tabref{tab:def} are pseudo-metrics:
\begin{proof}
This is trivially verified. For example, for the distance $d_{\gP}(f+g, f+g')$ given in \Tabref{tab:def}, which is the distance between $f+g, f+g'\in f+\gGh$, it is easy to check that the following properties do hold:
\begin{itemize}
    \item $d_{\gP}(f+g, f+g') = 0$ (subtraction of same functions evaluated on same $x$ and $y$)
    \item $d_{\gP}(f+g, f+g') = d_{\gP}(f+g', f+g)$ (evenness of absolute value)
    \item $d_{\gP}(f+g, f+g') \leq d_{\gP}(f+g, f+g'') + d_{\gP}(f+g'', f+g')$ (triangular inequality of absolute value)
\end{itemize}
Other distances in \Tabref{tab:def} can be proven to be pseudo-metrics in the same way.
\end{proof}

\begin{table}[t]
    \centering
    \caption{Capacity definitions of different sets by covering number with associated metric or pseudo-metric.}
    \resizebox{\textwidth}{!}{
    \begin{tabular}{lm{9cm}l}
    \toprule
    Capacity & Metric or pseudo-metric & Mentioned in \\
    \midrule
    \parbox{2.8cm}{$\gC(\varepsilon, \sH^{m}):=$\\
    $\sup_{\bm{\gP}} \gN(\varepsilon, \sH^{m}, d_{\bm{\gP}})$}
      & $
        d_{\bm{\gP}}((f+g_1,\dots,f+g_m), (f'+g_1',\dots,f'+g_m')) =\int_{(\gA\times T\gA)^{m}} \frac{1}{m}| \sum_{e\in E}\|(f+g_e)(x^e)-y^e\|^2 - \sum_{e\in E} \|(f'+g_e')(x^e)- y^e\|^2|\diff{\bm{\gP}}(\rmx, \rmy)$ & Theorem~\ref{thm:thm_baxter}; Prop. \ref{thm:thm_supp}\\
        \midrule
    \parbox{2.8cm}{$\gC_{\gGh}(\varepsilon, {\gFh}):=$\\$\sup_{\gP}\gN(\varepsilon, {\gFh}, d_{[\gP,{\gGh}]})$} & $d_{[\gP,{\gGh}]}(f, f') = \int_{{\gA\times T\gA}} \sup_{g\in {\gGh}} |\|(f+g)(x)-y\|^2-\|(f'+g)(x)-y\|^2|\diff \gP(x, y) $ & Prop. \ref{prop:sample_complexity_all_envs}, \ref{thm:thm_supp}, \ref{prop:sample_complexity_r}; Cor. \ref{prop:cor}\\\midrule
    \parbox{2.8cm}{$\gC_{\gFh}(\varepsilon, {\gGh}):=$\\$\sup_{\gP}\gN(\varepsilon, {\gGh}, d_{[\gP,{\gFh}]})$} & $d_{[\gP,{\gFh}]}(g, g') = \int_{{\gA\times T\gA}} \sup_{f\in {\gFh}} |\|(f+g)(x)-y\|^2-\|(f+g')(x)- y\|^2|\diff \gP(x, y) $ & Prop. \ref{prop:sample_complexity_all_envs}, \ref{thm:thm_supp}, \ref{prop:sample_complexity_RK} \\\midrule
    \parbox{2.8cm}{$\gC(\varepsilon, {f+\gGh}):=$\\ $\sup_{\gP}\gN(\varepsilon, {f+\gGh}, d_{\gP})$} & $d_{\gP}(f+g, f+g') = \int_{{\gA\times T\gA}} |\|(f+g)(x)- y\|^2-\|(f+g')(x)-y\|^2|\diff \gP(x, y) $ & Prop. \ref{prop:sample_complexity_after_training}\\
    \midrule
    \parbox{2.8cm}{$\gC(\varepsilon, {\gGh}, L^1):=$\\ $\sup_{\gP}\gN(\varepsilon, {\gGh}, d_{L^1(\gP)})$} & $d_{L^1(\gP)}(g, g')=\int_{\R^d} \|(g-g')(x)\|_1\diff \gP(x)$ & Prop. \ref{prop:sample_complexity_RK}; Theorem~\ref{thm:haussler}\\
    \midrule
    \parbox{2.8cm}{$\gC(\varepsilon, {\gGh}, L^2):=$\\ $\sup_{\gP}\gN(\varepsilon, {\gGh}, d_{L^2(\gP)})$} & $d_{L^2(\gP)}(g, g')=\sqrt{\int_{\R^d} \|(g-g')(x)\|^2_2\diff \gP(x)}$ & Prop. \ref{prop:sample_complexity_linear}; Lemma~\ref{lemma:bartlett} \\
    \bottomrule
    \end{tabular}
    }
    \label{tab:def}
\end{table}

\subsection{General Case\label{supp:sample_complexity_general}}

\subsubsection{\label{supp:sample_complexity_all_envs}Proof of Proposition~\ref{prop:sample_complexity_all_envs}}

\proptwo*
\begin{proof}
We introduce some extra definitions that are necessary for proving the proposition. Let $\gH = f+\gGh$ defined for each $f\in \gFh$, and let us define the product space $\gH^m = \{(f+g_1, \dots, f+g_m): f + g_e \in \gH\}$. Functions in this hypothesis space all have the same $f$, but not necessarily the same $g_e$. Let $\sH$ be the collection of all hypothesis spaces $\gH = f+\gGh, \forall f\in\gFh$. The hypothesis space associated to multiple environments is then defined as ${\sH^{m} := \bigcup_{\gH \in \sH}\gH^m}$.

Our proof makes use of two intermediary results addressed in Theorem~\ref{thm:thm_baxter} and Prop.~\ref{thm:thm_supp}.

\begin{theorem}[\protect\citeS{Baxter2000S}, Theorem 4, adapted to our setting]\label{thm:thm_baxter}
Assuming $\sH$ is a permissible hypothesis space family. For all $\varepsilon > 0$, if the number of examples $n$ of each environment  satisfies:
\begin{equation*}
n\geq \max\left\{\frac{64}{m\eps^2}\log \frac{4\gC(\frac{\eps}{16}, \sH^m)}{\delta}, \frac{16}{\varepsilon^2}\right\}
\end{equation*}
Then with probability at least $1-\delta$ (over the choice of $\{\gPh_e\}$), any $(f+g_1, \dots, f+g_m)$
will satisfy 
\[\frac{1}{m}\sum_{e\in E}\er_{\gP_e}(f+g_e) \leq  \frac{1}{m}\sum_{e\in E}\hat{\er}_{\gPh_e}(f+g_e) + \varepsilon\]
\end{theorem}

Note that permissibility (as defined in \citeS{Baxter2000S}) is a weak measure-theoretic condition satisfied by many real world hypothesis space families \citeS{Baxter2000S}. We will now express the capacity of $\sH^m$ in terms of the capacities of its two constituent component-spaces $\gFh$ and $\gGh$, thus leading to the main result.

\begin{prop}\label{thm:thm_supp}
For all $\varepsilon, \varepsilon_1, \varepsilon_2 > 0$ such that $\varepsilon=\varepsilon_1+\varepsilon_2$,
\begin{equation}
    \log\gC(\varepsilon, \sH^{m}) \leq\log\gC_{{\gGh}}(\varepsilon_1, {\gFh})+m\log\gC_{\gFh}(\varepsilon_2, {\gGh}) \label{eqn:t1-a}
\end{equation}
\end{prop}
\begin{subproof}[Proof of Proposition \ref{thm:thm_supp}]
To prove the proposition it is sufficient to show the property of covering sets for any joint distribution defined on all environments $\bm{\gP}$ on the space $(\gA\times T\gA)^{m}$. Let us then fix such a distribution $\bm{\gP}$. and let $\bar{\gP}=\frac{1}{m}\sum_{e\in E}\gP_e$ be the average distribution. 

Suppose that $F$ is an $\varepsilon_1$-cover of $({\gFh} ,d_{[\bar\gP,{\gGh}]})$ and $\{G_e\}_{e\in E}$ are $\varepsilon_2$-covers of $({\gGh},d_{[\gP_e,{\gFh}]})$. 
Let $H = \{(x_1,\dots,x_{m})\mapsto {((f+g_1)(x_1), \dots,(f+g_{m})(x_{m}))}, \: f\in F, \: g_e\in G_e\}$, be a set built from the covering sets aforementioned. Note that by definition $\lvert H\rvert=\lvert F\rvert\cdot\prod_{e\in E}\lvert G_e\rvert{{}\leq\gC_{{\gGh}}(\varepsilon_1, {\gFh})\gC_{\gFh}(\varepsilon_2, {\gGh})^{m}}$ as we take some distribution instances.

For each learner $(f+g_1, \dots, f+g_m) \in \sH^{m}$ in the hypothesis space, we take any $f'\in F$ such that $d_{[\bar \gP,{\gGh}]}(f, f'){}\leq\varepsilon_1$ and $g_e'\in G_e$ for all $e$ such that $d_{[\gP_e,{\gFh}]}(g_e, g_e'){}\leq\varepsilon_2$, and we build $(f'+g_1',\dots,f'+g_m')$. The distance is then:

\allowdisplaybreaks
\begin{align*}
    & d_{\bm{\gP}}((f+g_1, \dots, f+g_m), (f'+g_1', \dots, f'+g_m'))\\
    \leq{} & d_{\bm{\gP}}((f+g_1, \dots, f+g_m), (f'+g_1, \dots, f'+g_m)) \\
    & + d_{\bm{\gP}}((f'+g_1, \dots, f'+g_m), (f'+g_1', \dots, f'+g_m'))  \tag{triangular inequality of pseudo-metric}\\
    \leq{} & \frac{1}{m}\left[\sum_{e\in E}d_{\gP_e}(f+g_e, f'+g_e) + \sum_{e\in E}d_{\gP_e}(f'+g_e, f'+g_e') \right] \tag{triangular inequality of absolute value}\\
    \leq{} & \frac{1}{m}\sum_{e\in E}d_{[{\gP_e},{\gGh}]}(f, f') + \frac{1}{m}\sum_{e\in E}d_{[{\gP_e}, {\gFh}]}(g_e, g_e') \tag{by definition of $d_{[\gP_e,\gGh]}$ and $d_{[\gP_e,\gFh]}$} \\
    ={} & d_{[{\bar \gP},{\gGh}]}(f, f') + \frac{1}{m}\sum_{e\in E}d_{[{\gP_e}, {\gFh}]}(g_e, g_e') {}\leq\varepsilon_1 + \varepsilon_2 \tag{mean of the distance on different $\gP_e$ is the distance on $\bar\gP$ }
\end{align*}
To conclude, for any distribution $\bm{\gP}$, when $F$ is an $\varepsilon_1$-cover of $\gFh$ and $\{G_e\}$ are $\varepsilon_2$-covers of $\gGh$, the set $H$ built upon them is an $(\varepsilon_1+\varepsilon_2)$-cover of $\sH^{m}$. Then if we take the maximum over all distributions we conclude that $\gC(\varepsilon_1+\varepsilon_2, \sH^{m}){}\leq\gC_{{\gGh}}(\varepsilon_1, {\gFh})\gC_{\gFh}(\varepsilon_2, {\gGh})^{m}$ and we have \Eqref{eqn:t1-a}. 
\end{subproof}

We can now use the bound developed in Prop.~\ref{thm:thm_supp} and use it together with Theorem~\ref{thm:thm_baxter}, therefore concluding the proof of Prop.~\ref{prop:sample_complexity_all_envs}.
\end{proof}

\subsubsection{\label{supp:sample_complexity_after_training}Proof of Proposition~\ref{prop:sample_complexity_after_training}}

\propthree*
\begin{proof} The proof is derived from the following theorem which can be easily adapted to our context:
\begin{theorem}[\protect\citeS{Baxter2000S}, Theorem 3]\label{thm:thm_baxter_2}
Let $\gH$ a permissible hypothesis space. For all $0<\varepsilon, \delta< 1$, if the number of examples $n$ of each environment  satisfies:
\begin{equation*}
n\geq \max\left\{\frac{64}{m\eps^2}\log \frac{4\gC(\frac{\eps}{16}, \gH)}{\delta}, \frac{16}{\varepsilon^2}\right\}
\end{equation*}
Then with probability at least $1-\delta$ (over the choice of dataset $\gPh$ sampled from $\gP$), any $h\in \gH$
will satisfy 
\[\er_{\gP}(h) {}\leq \hat{\er}_{\gPh}(h) + \varepsilon\]
\end{theorem}

Given that $\gPh_{e'}$ is sampled from the same environment distribution $Q$, then by fixing the pre-trained $f$, we fix the space of hypothesis to $f + \gGh$, and we apply the Theorem~\ref{thm:thm_baxter_2} to obtain the proposition.
\end{proof}

\subsection{Linear case\label{supp:bound-linear}}
We provide here the proofs of theoretical bounds given in \Secref{sec:linear-case}. See the description in \Suppref{supp:extra-exp} for the detailed information on the example linear ODE dataset and the training with varying number of environments.

\subsubsection{Proof of Proposition 4}
\propfour*
\begin{proof}
Let us take $G$ an $\frac{\varepsilon}{2\sqrt{c}}$-cover of ${\gGh}$ with $L^2$-distance: $d_{L^2(\gP)}$ (see definition in \Tabref{tab:def}). Therefore, for each $\Lm_\mG\in{\gGh}$ take $g'\in G$ such that $d_{L^2}(\Lm_\mG, \Lm_{\mG'}){}\leq\frac{\varepsilon}{2\sqrt{c}}$, then
\begin{align*}
    &d_{[\gP,\gFh]}(\Lm_\mG, \Lm_{\mG'}) \\ 
    =&
        \int\limits_{\gA\times \gA'} \sup_{\Lm_\mF \in \gFh} |\|(\mF + \mG)x - y\|_2^2 - \|(\mF + \mG')x - y\|_2^2|
        \diff \gP(x, y) \\
    \leq & \int\limits_{\gA\times T\gA} \sup_{\Lm_\mF \in \gFh} \|(\mG-\mG')x\|_2 (\|(\mF + \mG)x - y\|_2 + \|(\mF+\mG')(x)-y\|_2) \diff \gP(x, y) \\
    \leq & \sqrt{\int\limits_{\gA} \|(\mG-\mG')x\|_2 \diff \gP(x)}  \sqrt{\int\limits_{\gA\times T\gA} \sup_{\Lm_\mF \in \gFh}(\|(\mF + \mG)x - y\|_2 + \|(\mF+\mG')x-y\|_2)^2 \diff \gP(x, y)} \\
    \leq & 2\sqrt{c} \sqrt{\int_{\R^d} \|(\mG-\mG')x\|_2\diff \gP(x)} \leq \varepsilon
\end{align*}
We have the $\gC_\gF(\varepsilon, {\gGh}){}\leq\gC(\frac{\varepsilon}{2\sqrt{c}}, {\gGh}, L^2)$. According to the following lemma:
\begin{lemma}[\protect\citeS{Bartlett2017S}, Lemma 3.2, Adapted]\label{lemma:bartlett} Given positive reals $(a, b, \eps)$ and positive integer $d$. Let vector $x \in \R^{d}$ be given with $\|x\|_p{}\leq b$, $\gGh = \{\Lm_{\mG}:\mG\in \R^{d\times d}, \|\mG\|_F^2{}\leq r\}$  where $\|\cdot\|_F$ is the Frobenius norm. Then
\[\log \mathcal{C}(\eps,\gGh,L^2)\leq \ceil{\frac{rdb^2}{\eps^2}}\log 2d^2\]
\end{lemma}
And we obtain that
\[
\log \gC_\gFh(\eps, {\gGh})\leq \left\lceil\frac{rcd(2b)^2}{\eps^2}\right\rceil\log2d^2 =: \omega(r,\eps)
\]
where $\omega(r,\eps)$ is a strictly increasing function \wrt $r$.
\end{proof}
\subsubsection{Proof of Proposition 5}
\propfive*
\begin{proof}
This can be derived from Prop.~\ref{prop:sample_complexity_all_envs} with the help of Prop.~\ref{prop:sample_complexity_linear} for linear maps. If we take the lower bounds of two capacities $\log \gC_\gFh(\frac{\eps_1}{16}, {\gGh})$ and $\log \gC_\gGh(\frac{\eps_2}{16}, {\gFh})$ for the linear maps hypothesis spaces $\gFh, \gGh$, then the number of required samples per environment $n$ now can be expressed as follows:
 \[n= \max\left\{\frac{64}{\eps^2}\left(\frac{1}{m} \log \frac{4}{\delta}+\frac{1}{m}\left\lceil\frac{r'cd(32b)^2}{\eps_1^2}\right\rceil\log2d^2 + \left\lceil\frac{rcd(32b)^2}{\eps_2^2}\right\rceil\log2d^2\right), \frac{16}{\eps^2}\right\}\]
To simplify the resolution of the equation above, we take $\eps_1=z\eps$ for any $0<z<1$, then $\eps_2=\eps-\eps_1=(1-z)\eps$.
Then by resolving the equation, the generalization margin is then upper bounded by $\eps$ with:
\[\eps = \max\left\{\sqrt{\frac{p+\sqrt{p^2+4q}}{2}}, \sqrt{\frac{16}{n}}\right\}\]
where $p = \frac{64}{mn}\log\frac{4}{\delta}$ and $q = \frac{64}{n}\left\lceil\left(\frac{r}{mz^2}+\frac{r'}{(1-z)^2}\right)  cd(32b)^2\right\rceil\log2d^2$.
\end{proof}

\subsection{Nonlinear case: instantiation for neural networks\label{supp:sample_complexity_nn}}
We show in this section how we design a concrete model for nonlinear dynamics following the general guidelines given in \Secref{subsec:general_case}. This is mainly composed of the following two parts: (a) choosing an appropriate approximation space and (b) choosing a penalization function $\Omega$ for this space. It is important to note that, even if the bounds given in the following sections may be loose in general, it could provide useful intuitions on the design of the algorithms which can be validated by experiments in our case.

\subsubsection{Choosing approximation space $\gFh$} We choose the space of feed-forward neural networks with a fixed architecture. Given the universal approximation properties of neural networks \citeS{KidgerL2020S}, and the existence of efficient optimization algorithms \citeS{ChizatB2018S}, this is a reasonable choice, but other families of approximating functions could be used as well. 

We then consider the function space of neural networks with $D$-layers with inputs and outputs in $\R^d$: $\NN=\{\nu: x \mapsto\sigma_D(W_D \cdots\sigma_1(W_1x))): x, \nu(x) \in \R^d \}$, $D$ is the depth of the network, $\sigma_j$ is a Lipschitz activation function at layer $j$, and $W_j$ weight matrix from layer $j-1$ to $j$. The number of adjustable parameters is fixed to $W$ for the architecture. This definition covers fully connected NNs and convolutional NNs. Note that the Fourier Neural Operator~\citeS{LiKALBSA2021S} used in the experiments for NS can be also covered by the definition above, as it performs alternatively the convolution in the Fourier space.

\subsubsection{Choosing penalization $\Omega$} Now we choose an $\Omega$ for the space above. Let us first introduce a practical way to bound the capacity of $\gGh\in \NN$. Proposition \ref{prop:sample_complexity_RK} tells us that for a fixed NN architecture (implying constant parameter number $W$ and depth $D$), we can control the capacity through the maximum output norm $R$ and Lipschitz norm $L$ defined in the proposition. 

\begin{prop}
\label{prop:sample_complexity_RK}
If for all neural network $g\in \gGh$,  $\|g\|_\infty{}= \ess \sup |g|{}\leq R$ and $\|g\|_{\mathrm{Lip}}{}\leq L$, with $\|\cdot\|_{\mathrm{Lip}}$ the Lipschitz semi-norm, then: 
\begin{equation}
    \log \gC_{\gFh}(\varepsilon, {\gGh})\leq
\omega(R,L,\eps)
\label{eq:omega1}
\end{equation}
where $\omega(R,L,\eps) = c_1\log\frac{RL}{\eps} + c_2$ for $c_1 = 2W$ and $c_2 = 2W\log 8e\sqrt{c}D$, with $c$ the bound of MSE loss. $\omega(R,L,\eps)$ is a strictly increasing function \wrt $R$ and $L$.
\end{prop}
\begin{proof} 
To link the capacity to some quantity that can be optimized for neural networks, we need to apply the following theorem:
\begin{theorem}[\protect\citeS{Haussler1992S}, Theorem 11, Adapted] \label{thm:haussler}
With the neural network function space $\NN$, let $W$ be the total number of adjustable parameters, $D$ the depth of the architecture. Let $\gGh\subseteq\NN$ be all functions into $[-R, R]^d$ representable on the architecture, and all these functions are at most $L$-Lipschitz. Then for all $0<\eps<2R$,
\[
\gC(\eps,\gGh, L^1) \leq \left(\frac{2e\cdot 2R\cdot DL}{\eps}\right)^{2W}
\]
\end{theorem}
Here, we need to prove firstly that the $\gFh$-dependent capacity of ${\gGh}$ is bounded by a scaled independent capacity on $L^1$ of itself. We suppose that the MSE loss function (used in the definitions in \Tabref{tab:def}) is bounded by some constant $c$. This is a reasonable assumption given that the input and output of neural networks are bounded in a compact set. Let us take $G$ an $\frac{\varepsilon}{2\sqrt{c}}$-cover of ${\gGh}$ with $L^1$-distance: $d_{L^1(\gP)}$ (see definition in \Tabref{tab:def}). Therefore, for each $g\in{\gGh}$ take $g'\in G$ such that $d_{L^1}(g, g')\leq\frac{\varepsilon}{2\sqrt{c}}$, then
\begin{align*}
    d_{[\gP,\gFh]}(g, g') =&
        \int\limits_{\gA\times \gA'} \sup_{f\in \gFh} |\|(f + g)(x) - y\|_2^2 - \|(f + g')(x) - y\|_2^2|\diff \gP(x, y) \\
    \leq & \int\limits_{\gA\times T\gA} \sup_{f\in \gFh} \|(g-g')(x)\|_2 (\|(f + g)(x) - y\|_2 + \|(f+g')(x)-y\|_2) \diff \gP(x, y)\\
    \leq & 2\sqrt{c} \int_{\R^d} \|(g-g')(x)\|_1\diff \gP(x) \leq \varepsilon
\end{align*}
Then we have the first inequality $\gC_\gF(\varepsilon, {\gGh})\leq\gC(\frac{\varepsilon}{2c}, {\gGh}, L^1)$. As we suppose that $\|g\|_\infty \leq R$ for all $g\in \gGh$, then for all $g \in \gGh$, we have $g(x)\in [-R,R]^d$. We now apply the Theorem~\ref{thm:haussler} on $\gGh$, we then have the following inequality 
\begin{equation}
    \log\gC\left(\frac{\varepsilon}{2\sqrt{c}}, {\gGh}, L^1\right)\leq2W\log\frac{8e\sqrt{c}DRL}{\varepsilon}
\end{equation} where $e$ is the base of the natural logarithm, $W$ is the number of parameters of the architecture, $D$ is the depth of the architecture. Then if we consider $W,c,D$ as constants, the bound becomes:
\begin{equation}
    \log\gC\left(\frac{\varepsilon}{2\sqrt{c}}, {\gGh}, L^1\right)\leq c_1\log\frac{RL}{\varepsilon} + c_2 = \omega(R,L,\eps)
\end{equation}
for $c_1 = 2W$ and $c_2 = 2W\log 8e\sqrt{c}D$.
\end{proof}
This leads us to choose for $\Omega$ a strictly increasing function that bounds $\omega(R,L,\eps)$. Given the inequality (\Eqref{eq:omega1}), this choice for $\Omega$ will allow us to bound practically the capacity of $\gGh$. 

Minimizing $\Omega$ will then reduce the effective capacity of the parametric set used to learn $g_e$.
Concretely, we choose for $\Omega$:
\begin{equation}
    \Omega(g_e) = \|g_e\|_\infty^2 + \alpha \|g_e\|^2_{\mathrm{Lip}}\tag{7}
\end{equation}
where $\alpha>0$ is a hyper-parameter. This function is strictly convex and attains its unique minimum at the null function.

With this choice, let us instantiate Prop. \ref{prop:sample_complexity_all_envs} for our familly of NNs. Let $r = \sup_{g\in\gGh}\Omega(g)$, and $\omega(r, \eps) = c_1\log\frac{r}{\eps\sqrt{\alpha}} + c_2$ (strictly increasing \wrt the $r$) for given parameters ${c_1,c_2>0}$. We have:
\begin{prop}
\label{prop:prop_nn_gen}
If $r = \sup_{g\in\gGh}\Omega(g)$ is finite, the number of samples $n$ \label{prop:sample_complexity_r} in \Eqref{eq:known-envs}, required to satisfy the error bound in Proposition \ref{prop:sample_complexity_all_envs} with the same $\delta,\eps,\eps_1$ and $\eps_2$ becomes:
\begin{equation}
n\geq \max\left\{\frac{64}{\eps^2}\left(\frac{1}{m}\log\frac{4\gC_\gGh(\frac{\eps_1}{16}, \gFh)}{\delta}\! + \omega\left(\!r,\frac{\varepsilon_2}{16}\right)\! \right), \frac{16}{\varepsilon^2}\!\right\}
\end{equation}
\end{prop}
\begin{proof}
If $\Omega(g_e)\leq r$, we have $2\log R \leq \log r$ and $ 2\log L + \log\alpha \leq \log r$, then \[\log RL \leq \log \frac{r}{\sqrt{\alpha}}\] 
We can therefore bound $\omega(R,L,\eps)$ by \begin{equation*}
    \omega(R,L,\eps) = c_1\log \frac{RL}{\eps} + c_2
    \leq c_1 \log\frac{r}{\eps\sqrt{\alpha}} + c_2 = \omega(r,\eps)
\end{equation*}
The result follows from Proposition \ref{prop:sample_complexity_RK}. 
\end{proof}
This means that the number of required samples will decrease with the size the largest possible $\Omega(g) = r$. The optimization process will reduce $\Omega(g_e)$ until a minimum is reached. The maximum size of the effective hypothesis space is then bounded and decreases throughout training. In particular, the following result follows:
\begin{cor}
Optimizing \Eqref{eq:lagrangian} for a given $\lambda$, we have that the number of samples $n$ in \Eqref{eq:known-envs} required to satisfy the error bound in Proposition \ref{prop:sample_complexity_all_envs} with the same $\delta,\eps,\eps_1$ and $\eps_2$ is:
\begin{equation}
n\geq\! \max\left\{\!\frac{64}{\eps^2}\left(\!\frac{1}{m}\log\frac{4\gC_\gGh(\frac{\eps_1}{16}, \gFh)}{\delta}\! + \omega\left(\!\lambda \kappa,\frac{\varepsilon_2}{16}\!\right)\!\right), \frac{16}{\varepsilon^2}\!\right\}
\end{equation}

where $\kappa=\sum_{e\in E} \sum_{i=1}^{l}\int_0^T\left\|\frac{\mathrm{d}x_s^{e,i}}{\mathrm{d}t}\right\|^2 \mathrm{d}s$.
\label{prop:cor}
\end{cor}
\begin{proof}
Denote $\gL_\lambda(f,\{g_e\})$ the loss function defining \Eqref{eq:lagrangian}. Consider a minimizer $(f^\star, \{g_e^\star\})$ of $\gL_\lambda$. Then:
$$\gL_\lambda(f^\star,\{g_e^\star\})\leq\gL_\lambda(0,\{0\})=\kappa$$
which gives:
$$\forall e,\ \Omega(g^\star_e)\leq\sum_e\Omega(g^\star_e)\leq\lambda \kappa$$
Defining $\gGh=\{g\in\gFh\ |\ \Omega(g)\leq\lambda \kappa\}$, we then have that \Eqref{eq:lagrangian} is equivalent to:
\begin{equation}
\begin{split}
\min_{f\in\gFh,\{g_e\}_{e\in E}\in\gGh} ~~~ \sum_{e\in E}&\left( \frac{\Omega(g_e)}{\lambda} + \sum_{i=1}^{l}\int_0^T\left\|\frac{\mathrm{d}x_s^{e,i}}{\mathrm{d}t} - (f + g_e)(x^{e,i}_s)\right\|^2 \mathrm{d}s\right)
\end{split}
\end{equation}
and the result follows from Proposition~\ref{prop:prop_nn_gen}.
\end{proof}
We can then decrease the sample complexity in the chosen NN family by: 
\begin{enumerate*}[label=(\alph*)]
    \item increasing the number of training environments engaged in the framework, and
    \item decreasing $\Omega(g_e)$ for all $g_e$, with $\Omega(g_e)$ instantiated as in \Secref{subsec:general_case}.
\end{enumerate*}
$\Omega$ provides a bound based on the largest output norm and the Lipschitz constant for a family of NNs. The experiments (\Secref{sec:exp}) confirm that this is indeed an effective way to control the capacity of the approximating function family. Note that in our experiments, the number of samples needed in practice is much smaller than suggested by the theoretical bound.

\begin{table}[t!]
    \centering
    \caption{Details for the results of evaluation error in test on linear systems in \Figref{fig:linear_consistency} .}
    \resizebox{\textwidth}{!}{
    \begin{tabular}{clcccc}
    \toprule
    Samples/env. & \multicolumn{1}{c}{Method} & $m=1$ & $m=2$ & $m=4$ & $m=8$ \\
    \midrule
    \multirow{2}{*}{$n=2\cdot K$}
    & \LEADSNM & \multirow{2}{*}{8.13$\pm$5.56 e-2} & 6.81$\pm$4.44 e-2 & 4.92$\pm$4.26 e-2 & 4.50$\pm$3.10 e-2 \\
    & \LEADS (Ours) & & \textbf{5.11$\pm$3.20 e-2} & \textbf{3.93$\pm$2.88 e-2} & \textbf{2.10$\pm$0.96 e-2} \\ \midrule
    \multirow{2}{*}{$n=4\cdot K$} 
    & \LEADSNM  & \multirow{2}{*}{4.08$\pm$2.57 e-2}  & 3.96$\pm$2.56 e-2 & 3.10$\pm$2.08 e-2 & 2.23$\pm$1.44 e-2\\ 
    & \LEADS (Ours) &  & \textbf{2.74$\pm$1.96 e-2} & \textbf{1.61$\pm$1.24 e-2} & \textbf{1.02$\pm$0.74 e-2} \\
    \bottomrule
    \end{tabular}}
    \label{tab:grid-2}
\end{table}

\begin{table}[ht!]
    \centering
    \caption{Detailed results of evaluation error in test on LV systems for \Figref{fig:plot}. For the case of $m = 1$, all baselines except FT-RNN are equivalent to \OPE. The arrows indicate that the table cells share the same value. }
    \resizebox{\textwidth}{!}{
    \begin{tabular}{clcccc}
    \toprule
    Samples/env. & \multicolumn{1}{c}{Method} & $m=1$ & $m=2$ & $m=4$ & $m=8$ \\
    \midrule
    \multirow{6}{*}{$n=1\cdot K$} & \OFA & 7.87$\pm$7.54 e-3 & 0.22$\pm$0.06 & 0.33$\pm$0.06 & 0.47$\pm$0.04 \\
    & \OPE & 7.87$\pm$7.54 e-3 & \multicolumn{3}{c}{\rightarrowfill} \\
    & FT-RNN & 4.02$\pm$3.17 e-2 & 1.62$\pm$1.14 e-2 & 1.62$\pm$1.40 e-2 & 1.08$\pm$1.03 e-2 \\
    & FT-NODE & 7.87$\pm$7.54 e-3 & 7.63$\pm$5.84 e-3 & 4.18$\pm$3.77 e-3 & 4.92$\pm$4.19 e-3 \\
    & GBML-like & 7.87$\pm$7.54 e-3 & 6.32$\pm$5.72 e-2 & 1.44$\pm$0.66 e-1 & 9.85$\pm$8.84 e-3\\
    & \LEADS (Ours) & 7.87$\pm$7.54 e-3 & \textbf{3.65$\pm$2.99 e-3} & \textbf{2.39$\pm$1.83 e-3} & \textbf{1.37$\pm$1.14 e-3}  \\ \midrule
    \multirow{6}{*}{$n=2\cdot K$} & \OFA & 1.38$\pm$1.61 e-3 & 0.22$\pm$0.04 & 0.36$\pm$0.07  & 0.60$\pm$0.11 \\ 
    & \OPE  & 1.38$\pm$1.61 e-3 & \multicolumn{3}{c}{\rightarrowfill}\\
    & FT-RNN & 7.20$\pm$7.12 e-2 & 2.72$\pm$4.00 e-2 & 1.69$\pm$1.57 e-2 & 1.38$\pm$1.25 e-2\\
    & FT-NODE & 1.38$\pm$1.61 e-3 & 9.02$\pm$8.81 e-3 & 1.11$\pm$1.05 e-3 & 1.00$\pm$0.95 e-3\\
    & GBML-like & 1.38$\pm$1.61 e-3 & 9.26$\pm$8.27 e-3 & 1.17$\pm$1.09 e-2 & 1.96$\pm$1.95 e-2\\
    & \LEADS{} (Ours) & 1.38$\pm$1.61 e-3 & \textbf{8.65$\pm$9.61 e-4} & \textbf{8.40$\pm$9.76 e-4} & \textbf{6.02$\pm$6.12 e-4} \\\midrule
    \multirow{6}{*}{$n=4\cdot K$} & \OFA & 1.36$\pm$1.25 e-4 & 0.19$\pm$0.02  & 0.31$\pm$0.04 & 0.50$\pm$0.04 \\ 
    & \OPE{} & 1.36$\pm$1.25 e-4 &\multicolumn{3}{c}{\rightarrowfill}  \\
    & FT-RNN & 8.69$\pm$8.36 e-4 & 3.39$\pm$3.38 e-4 & 3.02$\pm$1.50 e-4 & 2.26$\pm$1.45 e-4\\
    & FT-NODE & 1.36$\pm$1.25 e-4 & 1.74$\pm$1.65 e-4 & 1.78$\pm$1.71 e-4 & 1.39$\pm$1.20 e-4\\
    & GBML-like & 1.36$\pm$1.25 e-4 & 2.57$\pm$7.18 e-3 & 2.65$\pm$3.26 e-3 & 2.36$\pm$3.58 e-3 \\
    & \LEADS{} (Ours) & 1.36$\pm$1.25 e-4 & \textbf{1.10$\pm$0.92 e-4} & \textbf{1.03$\pm$0.98 e-4}  & \textbf{9.66$\pm$9.79 e-5} \\\midrule
    \multirow{6}{*}{$n=8\cdot K$} & \OFA & 5.98$\pm$5.13 e-5 & 0.16$\pm$0.03 & 0.35$\pm$0.06 & 0.52$\pm$0.06 \\ 
    & \OPE & 5.98$\pm$5.13 e-5 & \multicolumn{3}{c}{\rightarrowfill} \\
    & FT-RNN & 2.09$\pm$1.73 e-4 & 1.18$\pm$1.16 e-4 & 1.13$\pm$1.13 e-4 & 9.13$\pm$8.31 e-5\\
    & FT-NODE & 5.98$\pm$5.13 e-5 & 6.91$\pm$4.46 e-5 & 7.82$\pm$6.95 e-5 & 6.88$\pm$6.39 e-5 \\
    & GBML-like & 5.98$\pm$5.13 e-5 & 1.02$\pm$1.68 e-4 & 1.41$\pm$2.68 e-4 & 0.99$\pm$1.53 e-4\\
    & \LEADS{} (Ours) & 5.98$\pm$5.13 e-5 & \textbf{5.47$\pm$4.63 e-5} & \textbf{4.52$\pm$3.98 e-5}  & \textbf{3.94$\pm$3.49 e-5} \\
    \bottomrule
    \end{tabular}}
    \label{tab:grid}
\end{table}

\section{\label{supp:discussion} Optimizing $\Omega$ in practice}

In \Secref{subsec:nn_instantiation}, we developed an instantiation of the \LEADS framework for neural networks. We proposed to control the capacity of the $g_e$s components through a penalization function $\Omega$ defined as $\Omega(g_e) = \|g_e\|_\infty^2 + \alpha\|g_e\|_{\mathrm{Lip}}^2$. This definition ensures the properties required to control the sample complexity. 

However, in practice, both terms in $\Omega(g_e)$ are difficult to compute as they do not yield an analytical form for neural networks. For a fixed activation function, the Lipschitz-norm of a trained model only depends on the model parameters and, for our class of neural networks, can be bounded by the spectral norms of the weight matrices, as described in \Secref{subsec:training-details}. This allows for a practical implementation.

The infinity norm on its side depends on the domain definition of the function and practical implementations require an empirical estimate. Since there is no trivial estimator for the infinity norm of a function, we performed tests with different proxies such as the \textit{empirical} $L^p$ and $L^\infty$ norms, respectively defined as $\|g_e\|_{L^p(\gPh_e)}=\left(\frac{1}{n}\sum_{x\in \gPh_e}|g_e(x)|^p\right)^{1/p}$ for $1\leq p < \infty$ and $\|g_e\|_{L^\infty(\gPh_e)}=\max_{x\in \gPh_e}|g_e(x)|$. Here $|\cdot|$ is an $\ell^2$ vector norm.  Note that on a finite set of points, these norms reduce to vector norms $\|(|g_e(x_1)|, \dots, |g_e(x_n)|)^\top\|_p$. They are then all equivalent on the space defined by the training set. \Tabref{tab:lp-min} shows the results of experiments performed on LV equation with different $1\leq p\leq \infty$. Overall we found that $L^p$ for small values of $p$ worked better and chose in our experiments set $p=2$.
\begin{table}[t]
    \centering
    \caption{Test MSE of experiments on LV ($m=4, n=1\cdot K$) with different empirical norms.}
    \begin{tabular}{cccccc}
    \toprule
    Empirical Norm & $p=1$ & $p=2$ & $p=3$ & $p=10$ & $p=\infty$ \\
    \midrule
      Test MSE  &  2.30e-3 &  2.36e-3 &  2.34e-3 & 3.41e-3 & 6.12e-3 \\
     \bottomrule
    \end{tabular}
    \label{tab:lp-min}
\end{table}

Moreover, using both minimized quantities $\|g_e\|^2_{L^2(\gPh_e)}$ and the spectral norm of the product of weight matrices, denoted $L(g_e)$ and $\Pi(g_e)$ respectively, we can give a bound on $\Omega(g_e)$. First, for any $x$ in the compact support of $\gP_e$, we have that, fixing some $x_0\in\gPh_e$:
\[
|g_e(x)|{}\leq |g_e(x)-g_e(x_0)|{}+{}|g_e(x_0)|
\]
For the first term:
\[
|g_e(x)-g_e(x_0)|{}\leq \|g_e\|_{\mathrm{Lip}}|x-x_0|{}\leq \Pi(g_e) |x-x_0|
\]
and the support of $\gP_e$ being compact by hypothesis, denoting by $\delta$ its diameter:
\[
|g_e(x)-g_e(x_0)|{}\leq \delta\Pi(g_e)
\]
Moreover, for the second term:
\[
|g_e(x_0)|{}= \sqrt{|g_e(x_0)|^2} \leq \sqrt{L(g_e)} 
\]
and summing both contributions gives us the bound:
\[
\|g_e\|_\infty{}\leq \delta\Pi(g_e) + \sqrt{L(g_e)}
\]
so that:
\[
\Omega(g_e){}\leq (\delta+\alpha)\Pi(g_e) + \sqrt{L(g_e)}
\]
Note that this estimation is a crude one and improvements can be made by considering the closest $x_0$ from $x$ and taking $\delta$ to be the maximal distance between points not from the support of $\gP_e$ and $\gPh_e$.

Finally, we noticed that minimizing $\|\frac{g_e}{id}\|^2_{L^2(\gPh_e)}$ in domains bounded away from zero gave better results as normalizing by the norm of the output allowed to adaptively rescale the computed norm. Formally, minimizing this quantity does not fundamentally change the optimization as we have that:
\[
\forall x, \frac{1}{M^2}|g_e(x)|^2\leq\left|\frac{g_e(x)}{x}\right|^2 \leq \frac{1}{m^2}|g_e(x)|^2
\]
meaning that:
\[
\frac{1}{M^2}L(g_e)\leq\left\|\frac{g_e}{id}\right\|^2_{L^2(\gPh_e)} \leq \frac{1}{m^2}L(g_e)
\]
where $m,M$ are the lower and upper bound of $|x|$ on the support of $\gP_e$ with $m>0$ by hypothesis~(the quantity we minimize is still higher than $L(g_e)$ even if this is not the case).

\section{Additional experimental details\label{supp:extra-exp}}

\subsection{Details on the environment dynamics}
\paragraph{Lotka-Volterra (LV).} 
The model dynamics follow the ODE:
\[
    \frac{\diff u}{\diff t} = \alpha u- \beta uv,  \nicefrac{\diff v}{\diff t} = \delta uv - \gamma v
\]
with $u, v$ the number of prey and predator, $\alpha,\beta,\gamma,\delta\!>\!0$ defining how the two species interact. The initial conditions $u_0^i, v_0^i$ are sampled from a uniform distribution $P_0=\operatorname{Unif}([1,2]^2)$. We characterize the dynamics by $\theta=(\sfrac{\alpha}{\beta},\sfrac{\gamma}{\delta}) \in \Theta=\{0.5,1,1.44,1.5,1.86,2\}^2$. An environment $e$ is then defined by parameters $\theta_e$ sampled from a uniform distribution over the parameter set $\Theta$.
\paragraph{Gray-Scott (GS).} The governing PDE is:
\begin{equation*}
    \frac{\partial u}{\partial t} = D_u\Delta u - uv^2 + F(1-u), \nicefrac{\partial v}{\partial t} = D_v\Delta v + uv^2 - (F+k)v
\end{equation*}
where the $u, v$ represent the concentrations of two chemical components in the spatial domain $S$ with periodic boundary conditions. $D_u, D_v$ denote the diffusion coefficients respectively for $u, v$, and are held constant to $D_u = 0.2097, D_v = 0.105$, and $F, k$ are the reaction parameters depending on the environment. As for the initial conditions $(u_0, v_0)\sim P_0$, we place 3 2-by-2 squares at uniformly sampled positions in $S$ to trigger the reactions. The values of $(u_0, v_0)$ are fixed to $(0,1)$ outside the squares and to $(1-\epsilon,\epsilon)$ with a small $\epsilon\!>\!0$ inside.
An environment  $e$ is defined by its parameters $\theta_e=(F_e, k_e)\in \Theta = \{(0.037,0.060), (0.030,0.062), (0.039,0.058)\}$. We consider a set of $\theta_e$ parameters uniformly sampled from the environment distribution $Q$ on $\Theta$.
\paragraph{Navier-Stokes (NS).} We consider the Navier-Stokes PDE for incompressible flows:
\[
\frac{\partial w}{\partial t} = - v\cdot\nabla w + \nu\Delta w + \xi \qquad \nabla\cdot v = 0
\]
where $v$ is the velocity field, $w=\nabla \times v$ is the vorticity, both $v, w$ lie in a spatial domain $S$ with periodic boundary conditions, $\nu$ is the viscosity and $\xi$ is the constant forcing term in the domain $S$. 
We fix $\nu=10^{-3}$ across the environments. We sample the initial conditions $w^e_0\sim P_0$ as in \citeS{LiKALBSA2021S}. An environment $e$ is defined by its forcing term $\xi_e\in \Theta_\xi=\{\xi_1, \xi_2, \xi_3, \xi_4\}$ with
\begin{align*}
    \xi_1(x,y) & = 0.1(\sin(2\pi(x + y)) + \cos(2\pi(x + y))) \\
    \xi_2(x,y) & = 0.1(\sin(2\pi(x + y)) + \cos(2\pi(x + 2y))) \\
    \xi_3(x,y) & = 0.1(\sin(2\pi(x + y)) + \cos(2\pi(2x + y))) \\
    \xi_4(x,y) & = 0.1(\sin(2\pi(2x + y)) + \cos(2\pi(2x + y)))
\end{align*}
where $(x,y)\in S$ is the position in the domain $S$. We uniformly sampled a set of forcing terms from $Q$ on $\Theta_\xi$.

\paragraph{Linear ODE.}
We take an example of linear ODE expressed by the following formula: \[\frac{\diff u_t}{\diff t}=\Lm_{\mQ\mLambda\mQ^\top}(u_t) = \mQ\mLambda\mQ^\top u_t\]
where $u_t\in \R^8$ is the system state, $\mQ\in M_{8,8}(\R)$ is an orthogonal matrix such that $\mQ\mQ^\top = 1$, and $\mLambda \in M_{8,8}(\R)$ is a diagonal matrix containing eigenvalues. We sample $\mLambda_e$ from a uniform distribution on $\Theta_\mLambda = \{\mLambda_1, \dots, \mLambda_8\}$, defined for each $\mLambda_i$ by:
\[
\left[\mLambda_{i}\right]_{jj} = 
\begin{cases}
0,   & \text{ if $i = j$ for $i, j \in \{1, \dots, 8\}$, } \\
-0.5,  & \text{ otherwise. }
\end{cases}
\]
which means that the $i$-th eigenvalue is set to 0, while others are set to a common value $-0.5$.

\subsection{Choosing hyperparameters} As usual, the hyperparameters need to be tuned for each considered set of systems. We therefore chose the hyperparameters using standard cross-validation techniques. We did not conduct a systematic sensitivity analysis. In practice, we found that: (a) if the regularization term is too large \wrt the trajectory loss, the model cannot fit the trajectories, and (b) if the regularization term is too small, the performance is similar to \LEADSNM{} The candidate hyperparameters are defined on a very sparse grid, for example, for neural nets, $(10^3, 10^4, 10^5, 10^6)$ for $\lambda$ and $(10^{-2}, 10^{-3}, 10^{-4}, 10^{-5})$ for $\alpha$.

\subsection{Details on the experiments with a varying number of environments}
We conducted large-scale experiments respectively for linear ODEs (\Secref{sec:linear-case}, \Figref{fig:linear_consistency}) and LV (\Secref{sec:exp}, \Figref{fig:plot}) to compare the tendency of \LEADS \wrt the theoretical bound and the baselines by varying the number of environments available for the instantiated model. 

To guarantee the comparability of the test-time results, we need to use the same test set when varying the number of environments. We therefore propose to firstly generate a global set of environments, separate it into subgroups for training, then we test these separately trained models on the global test set. 

We performed the experiments as follows:
\begin{itemize}
    \item In the training phase, we consider $M = 8$ environments in total in the environment set $E_{\text{total}}$. We denote here the cardinality of an environment set $E$ by $\operatorname{card}(E)$, the environments are then arranged into $b = 1, 2, 4$ or $8$ disjoint groups of the same size, \ie $\{E_1, \dots, E_b\}$ such that $\bigcup_{i=1}^b E_i = E_{\text{total}}$,  $\operatorname{card}(E_1){}=\dots=\operatorname{card}(E_b){}=\floor{M/b} =: m$, where $m$ is the number of environments per group, and $E_i\cap E_j=\emptyset$ whenever $i\neq j$. For example, for $m=1$, all the original environments are gathered into one global environment, when for $m=8$ we keep all the original environments. The methods are then instantiated respectively for each $E_i$. For example, for \LEADS with $b$ environment groups, we instantiate $\text{\LEADS}_1,\dots,\text{\LEADS}_b$ respectively on $E_1,\dots,E_b$. Other frameworks are applied in the same way. 
    
    Note that when $m=1$, having $b=8$ environment groups of one single environment, \OFA, \OPE and \LEADS are reduced to \OPE applied on all $M$ environments. We can see in \Figref{fig:plot} that each group of plots starts from the same point.
    
    \item In the test phase, the performance of the model trained with the group $E_i$ is tested with the test samples of the corresponding group. Then we take the mean error over all $b$ groups to obtain the results on all $M$ environments. Note that the result at each point in \twofigref{fig:linear_consistency}{fig:plot} is calculated on the same total test set, which guarantees the comparability between results.
\end{itemize}

\subsection{Additional experimental results}

\paragraph{Experiments with a varying number of environments} We show in \twotabref{tab:grid-2}{tab:grid} the detailed results used for the plots in \twofigref{fig:linear_consistency}{fig:plot}, compared to baseline methods.

\paragraph{Learning in novel environments} We conducted same experiments as in \Secref{subsec:experimental-results} to learn in unseen environments for GS and NS datasets. The test MSE at different training steps is shown in \Tabref{tab:train-new-env}.

\begin{table}
    \centering
    \caption{Results on 2 novel environments for LV, GS, and NS at different traning steps with $n$ data points per env. The arrows indicate that the table cells share the same value. }
    \setlength{\tabcolsep}{3pt}
    \begin{tabular}[c]{clccc}
    \toprule
    \multirow{2}{*}{Dataset} & \multirow{2}{*}{Training Schema} & \multicolumn{3}{c}{Test MSE at training step}  \\
    \cmidrule{3-5} 
    & & 50 & 2500 & 10000 \\
    \midrule 
    \multirow{3}{*}{LV ($n=1\cdot K$)} & \PRT & 0.36 & \multicolumn{2}{c}{\rightarrowfill} \\
    & \OPE from scratch &0.23& 8.85e-3 & 3.05e-3\\
    & \LEADSBS & 0.73 & \textbf{1.36e-3} & \textbf{1.11e-3} \\
    \midrule 
    \multirow{3}{*}{GS ($n=1\cdot K$)} & \PRT & 5.44e-3 & \multicolumn{2}{c}{\rightarrowfill}  \\
    & \OPE from scratch & 4.20e-2& 5.53e-3 & 3.05e-3\\
    & \LEADSBS & 2.29e-3 & \textbf{1.45e-3} & \textbf{1.27e-3} \\
    \midrule 
    \multirow{3}{*}{NS ($n=8\cdot K$)} & \PRT & 1.75e-1 & \multicolumn{2}{c}{\rightarrowfill} \\
    & \OPE from scratch & 6.76e-2 & 1.70e-2 & 1.18e-2\\
    & \LEADSBS & 1.37e-2 & \textbf{8.07e-3} & \textbf{7.14e-3} \\
    \bottomrule
    \end{tabular}
    \label{tab:train-new-env}
\end{table}


\paragraph{Full-length trajectories}

We provide in figures S1-S4 the full-length sample trajectories for GS and NS of \Figref{fig:result-gs-comp-pred}.

\begin{figure}[p]
    \centering
    \subfloat[\OPE]{\includegraphics[width=0.7\textwidth]{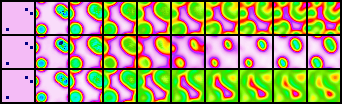}}
    
    \subfloat[FT-NODE]{\includegraphics[width=0.7\textwidth]{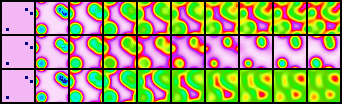}}
    
    \subfloat[\LEADS]{\includegraphics[width=0.7\textwidth]{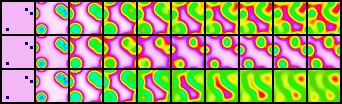}}
    
    \subfloat[Ground truth]{\includegraphics[width=0.7\textwidth]{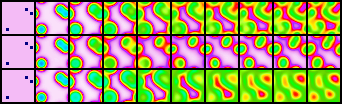}}
 
    \caption{Full-length prediction comparison of \Figref{fig:result-gs-comp-pred} for GS. In each figure, from top to bottom, the trajectory snapshots are output respectively from 3 training environments. The temporal resolution of each sequence is $\Delta t = 40$.}
    \label{fig:res-gs-full}
\end{figure}

\begin{figure}[p]
    \centering
    \subfloat[Difference between \OPE and Ground truth]{\includegraphics[width=0.7\textwidth]{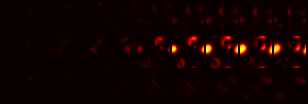}}
    
    \subfloat[Difference between FT-NODE and Ground truth]{\includegraphics[width=0.7\textwidth]{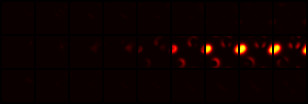}}
    
    \subfloat[Difference between \LEADS and Ground truth]{\includegraphics[width=0.7\textwidth]{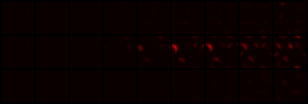}}
    \caption{Full-length error maps of \Figref{fig:result-gs-comp-pred} for GS. In each figure, from top to bottom, the trajectory snapshots correspond to 3 training environments, one per row. The temporal resolution of each sequence is $\Delta t = 40$.}
    \label{fig:res-gs-full-diff}
\end{figure}

\begin{figure}[p]
    \centering
    \subfloat[\OPE{}]{\includegraphics[width=0.7\textwidth]{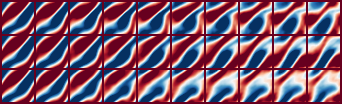}}
    
    \subfloat[FT-NODE]{\includegraphics[width=0.7\textwidth]{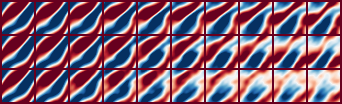}}   
    
    \subfloat[\LEADS{}]{\includegraphics[width=0.7\textwidth]{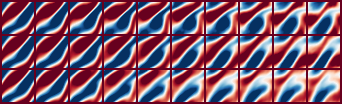}}
    
    \subfloat[Ground truth]{\includegraphics[width=0.7\textwidth]{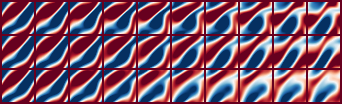}}
 
    \caption{Full-length prediction comparison of \Figref{fig:result-gs-comp-pred} for NS. In each figure, from top to bottom, the trajectory snapshots correspond to 3 training environments. The temporal resolution of each sequence is $\Delta t = 1$.}
    \label{fig:res-ns-full}
\end{figure}

\begin{figure}[p]
    \centering
    \subfloat[Difference between \OPE and Ground truth]{\includegraphics[width=0.7\textwidth]{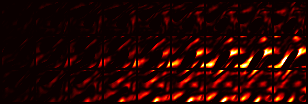}}
    
    \subfloat[Difference between FT-NODE and Ground truth]{\includegraphics[width=0.7\textwidth]{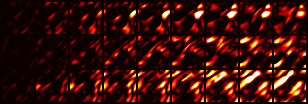}}
    
    \subfloat[Difference between \LEADS and Ground truth]{\includegraphics[width=0.7\textwidth]{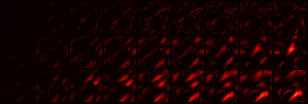}}
    \caption{Full-length error maps of \Figref{fig:result-gs-comp-pred} for NS. In each figure, from top to bottom, the trajectory snapshots correspond to from 3 training environments. The temporal resolution of each sequence is $\Delta t = 1$.}
    \label{fig:res-ns-full-diff}
\end{figure}

\clearpage
\bibliographyS{bibS}
\bibliographystyleS{abbrv}

\end{document}